\documentclass[11pt]{article}
\usepackage{epsfig}
\usepackage{color}
\usepackage{amsmath}
\usepackage{amssymb}  %
\usepackage{comment}
\usepackage{graphicx}
\usepackage{pgfplots}
\usepackage{mathdots}
\usepackage{tikz}
\usepackage{url}
\usepackage{multirow}
\usepackage{hyperref}
\usepackage{cleveref}
\usepackage{siunitx}
\usepackage{caption}
\usepackage{subcaption}
\usepackage{booktabs}
\usepackage{tabularx}
\usepackage{amsthm}

\newfont{\Bb}{msbm10 scaled\magstep0}

\newtheorem{theorem}{Theorem}[section]

\usepackage{algorithm}
\usepackage{algorithmic}
\usepackage[normalem]{ulem}
\useunder{\uline}{\ul}{}

\title{Spatially-Enhanced Temporal Fusion Transformer: Interpretable Multi-Output Prediction for Parametric Dynamical Systems with Time-Varying Inputs}
\author{Shuwen Sun\thanks{Corresponding author, Max Planck Institute for Dynamics of Complex Technical Systems, Germany. {\tt ssun@mpi-magdeburg.mpg.de}},
Lihong Feng\thanks{Max Planck Institute for Dynamics of Complex Technical Systems, Germany. {\tt feng@mpi-magdeburg.mpg.de}} 
and Peter Benner\thanks{Max Planck Institute for Dynamics of Complex Technical Systems, Germany and Fakultät für Mathematik, Otto-von-Guericke-Universität Magdeburg, Germany. {\tt benner@mpi-magdeburg.mpg.de}}
}

\begin{document}



\maketitle

\begin{abstract}
We explore the promising performance of a transformer model in predicting outputs of parametric dynamical systems with external time-varying input signals. The outputs of such systems vary not only with physical parameters but also with external time-varying input signals. Accurately catching the dynamics of such systems is challenging. We have adapted and extended an existing transformer model, called temporal fusion transformer (TFT), for single-output prediction to a multiple-output transformer, named as Spatially-Enhanced Temporal Fusion Transformer (SE-TFT), which is able to predict multiple output responses of these systems. The SE-TFT generalizes the interpretability of the original TFT model. The generalized interpretable attention weight matrix explores not only the temporal correlations in the sequence, but also the interactions between the multiple outputs, providing explanation for the spatial correlation in the output domain. This proposed SE-TFT accurately predicts the sequence of multiple outputs, regardless of the nonlinearity of the system and the dimensionality of the parameter space.
\end{abstract}

\section{Introduction}
With the increasing needs from various engineering fields, we are facing simulating more and more large-scale complex systems in the form of differential-algebraic equations (DAEs) or ordinary differential equations (ODEs) with a large number of degrees of freedoms (DOFs). Numerically solving such systems often takes too much time, especially when they need to be simulated repeatedly for many parameter instances or input-signal changes. In order to reduce the computational cost of simulating those large-scale systems, model order reduction (MOR) has been actively researched for more than 30 years for proposing surrogate models with much less DOFs~\cite{morAnt05, morBenetala21, morBenetalb21, morBenetalc21, morFelF95, morPilR90}. Consequently, such surrogate models can replace the original large-scale systems in many multi-query tasks to achieve fast computation. However, strongly nonlinear parametric systems with external time-varying input signals are still challenging for MOR. With the power of modern computers for processing large amounts of data, machine learning (ML) is being applied in computational science~\cite{morBarFM23, morBonCGetal24, morConGFetal23, morConGMetal23, morDuFPetal20, morFreDM21, morFreM22, morHeCFetal23, morDuaH24, morKasGH20, morKut21, morMau21}. The large-scale systems are being replaced with neural networks (NNs) as a new kind of surrogate models. Compared with traditional projection-based MOR methods~\cite{morAnt05, morBarF22, morBenetala21, morBenetalb21, morBenetalc21, morBenGW15, morFelF95, morPilR90} for surrogate modeling, ML models are non-intrusive and data-driven. Hence, they are efficient for developing surrogates for many complex mathematical models, for which the system matrices and the nonlinear vector resulting from spatial discretization, are hard to be explicitly extracted from simulation tools.

Many ML methods aim to accurately predict the whole solution vector, therefore autoencoders (AE) are often used to first compress the data of the numerical solution trajectories into a latent space with much lower dimension. Different data-driven methods, such as long short-term memory~(LSTM), dynamic mode decomposition~(DMD), sparse identification of non-linear dynamics~(SINDy), neural ordinary differential equations (NODEs), etc., are then used to learn the dynamics in the latent space~\cite{morBonCGetal24, mlCheRBetal18, morConGFetal23, morFreM22, morHeCFetal23, morDuaH24, morRicS, morSunFB26}. On the other hand, neural operators~\cite{mlliKAetal20b, mlliKAetal20a, mlluJPetal21} directly learn a mapping between the input function and the solution of partial differential equations (PDEs), so that they are independent of any discrete mesh used for numerically solving the PDEs.

Quantities of interests (QoIs) are usually a few scalar functions of the solution or state vector, which are sometimes sufficient for certain analyses. In system theory, QoIs are called the outputs of the dynamical systems. There are a few works focusing on predicting only the QoIs or outputs using ML methods, such as in~\cite{morFen23, MorSD21}, without employing data compression. In these studies, LSTM networks are used to predict parametric outputs that vary with external input signals. However, LSTM is known to suffer from issues with long-term predictions and slow inference~\cite{morFen23}. Moreover, the amount of the window data depends on the complexity of the problems. For some problems, the window must be taken to a bit larger for accurate prediction. The data in the window, nevertheless, need to be generated by simulating the original large-scale systems or by additional measurements.

The transformer models are proposed to overcome the difficulties of recurrent neural networks, such as LSTM, for long-term predictions. Many transformer models have been proposed for time series forecast, please see a recent survey~\cite{mlQinTZetal23} on various transformer models for different tasks, such as time-series prediction, classification, spatial-temporal prediction, etc. Most of the transformer models are applied to predict daily life activities, such as electricity consumption, traffic road occupancy rate, weather forecast, currency exchange rate, etc.~\cite{mlCirGYetal22, mlDroMC22, mlLimALetal21, mlLinKR21, mlLiuWWetal22, mlShaAMetal23, mlSheW22, mlZhaY23, mlZhoZPetal21, mlZhoMWetal22}. Some transformer models have been proposed for prediction based on numerical solution of large-scale dynamical systems via latent space dynamics learning or neural operator learning~\cite{mlCalKL24, mlGenZ22, mlHaoWS23, mlKoGKetal24, mlOvaKS24, morRicS}. In~\cite{mlGenZ22}, a transformer model was applied to construct a surrogate model of large physical dynamical systems, where a Koopman-based embeddings approach is proposed. The input of the model is the initial state and the trained model can predict the dynamics subsequently. In~\cite{morRicS}, only non-parametric dynamical systems are considered, and a transformer is used to learn the latent dynamics only in the time domain. To the best of our knowledge, few of those have yet been applied to predicting QoIs or outputs of large-scale dynamical systems with external inputs, in both the time and parameter domain. Note that in~\cite{mlWuG}, a small non-parametric dynamical system with three state variables describing the influenza-like illness symptoms is studied. The states of the system are predicted using a transformer model. However, neither parameters nor external inputs are considered in the system. A recent work using a transformer architecture for neural operator learning~\cite{mlHaoWS23} could also be applied for predicting QoIs. However, neural operator learning usually requires much more training data than other NNs, especially for tasks of predicting long-term sequences depending on multiple factors, e.g., parameters, external inputs, initial conditions, etc.

In order to predict time series dependent on static covariates, a priori known inputs, and observed inputs effectively, a transformer model: temporal fusion transformer (TFT) in~\cite{mlLimALetal21} is introduced. In this work, we propose to apply TFT to predict the time evolution of parametric outputs of dynamical systems with external input signals. It is shown in~\cite{mlLimALetal21} that TFT is accurate in long-term prediction of time series dependent on a complex mix of inputs, including time-invariant (static) covariates, known future inputs, and time series that are only observed in the past. TFT was used to predict the electricity usage, the traffic flow, etc., in a future time period~\cite{mlLimALetal21}. Translating these terminologies into the terms of system theory, we understand that the static covariates correspond to the time-independent physical/geometrical parameters, the known future inputs correspond to the time-varying external input signals, and the time-series that are only observed in the past are simply the outputs in the past time period. In summary, TFT should be able to predict outputs in both parameter and time domain (future time prediction), given the parameters, the input signals as well as the outputs in the past time period. In contrast to many existing autoregressive methods for time sequence prediction~\cite{mlCheRBetal18, morConGFetal23, morConGMetal23, morHeCFetal23, morMau21, morRicS}, which generate predictions step-by-step, TFT~\cite{mlLimALetal21} is able to do multi-horizon prediction in a single prediction pass. 

The current TFT model is limited to predicting a single QoI. To predict a different QoI, TFT must be re-trained. We propose a Spatially-Enhanced Temporal Fusion Transformer (SE-TFT) model, which aims to predict multiple QoIs with a single training session. A new masking scheme is proposed for the interpretable multi-head attention to illustrate the correlation between different outputs. In this sense, the spatial relations in the output domain are explored. The interpretability of the TFT is then naturally generalized by the SE-TFT, such that the resulting attention weight matrix in SE-TFT provides information on temporal and spatial interactions between features corresponding to different outputs. Note that several implementations of TFT~(for example, Pytorch Forecasting package~\cite{py_forecast}) have supported TFT for multiple targets predictions. However, these implementations use the same attention matrix as the original TFT, which has no interpretation of the correlation between the multiple targets (outputs). In contrast, our proposed method integrates spatio-temporal learning into the self-attention layer using a specially designed mask structure. As a result, the proposed SE-TFT not only captures spatio-temporal information but also enhances interpretability for multiple target predictions.

We have successfully applied SE-TFT to three parametric dynamical systems: The Lorenz-63 model parametrized with random initial conditions; the FitzHugh-Nagumo model with two physical parameters and a time-varying external input signal; a coupled electrochemical kinetics and diffusion model with two physical parameters and an external input signal changing with both the time and a parameter. The first model is used as a benchmark for chaotic dynamical systems. The second one has cubic nonlinearity. For the coupled electrochemical kinetics and diffusion model, all the equations are fully coupled, it is difficult to extract system matrices and nonlinear vectors that are necessary for projection-based intrusive MOR. The prediction results show that SE-TFT is very accurate for predicting multiple outputs in both the parameter and the time domain. After training, SE-TFT can accurately predict the outputs at any testing parameter samples in one step given only the solution at the initial time instance. The attention weight matrix for each testing case illustrates not only the temporal relationship within a single output sequence, but also the interactions between features corresponding to different outputs along the whole-time sequence, clearly showing the interpretability of SE-TFT. 

In the next section, we present the parametric dynamical systems under consideration, then the structure of TFT is introduced and is connected to the dynamical systems. Section~\ref{sec:MOTFT} demonstrates the framework of SE-TFT, the pre-processing, and the training procedure. Section~\ref{sec:pred_results} presents the prediction results of SE-TFT using the above mentioned three examples and the interpretability analysis of SE-TFT based on the numerical results. Section~\ref{sec:conclusion} concludes the article with some further discussions. 

\section{Parametric dynamical systems and structure of TFT}

\subsection{Parametric dynamical systems}
\label{sec:sys}
The parametric dynamical systems we consider are in the form of differential-algebraic equations (DAEs):
\begin{align}
	\label{eq:fom}
\frac{d}{dt} \boldsymbol{E}(\boldsymbol{\mu})\boldsymbol{x}(t, \boldsymbol{\mu}) &= F(\boldsymbol{x}(t, \boldsymbol{\mu}), \boldsymbol{\mu}) + \boldsymbol{B}(\boldsymbol{\mu}) \boldsymbol{u}(\boldsymbol{\mu}, t), \qquad \boldsymbol{x}(0, \boldsymbol{\mu}) = \boldsymbol{x}_{0}(\boldsymbol{\mu}), \nonumber \\
		\boldsymbol{y}(t, \boldsymbol{\mu}) &= \boldsymbol{C}(\boldsymbol{\mu}) \boldsymbol{x}(t, \boldsymbol{\mu}),
\end{align}
where $t \in [0, T]$ and $\boldsymbol{\mu}:=(\mu_1,\ldots,\mu_p)^T \in \mathcal{P} \subset \mathbb R^{p}$, $\mathcal{P}$ is the parameter domain. The unknown state vector $\boldsymbol{x}(\boldsymbol{\mu})\in \mathbb R^N$ and $\boldsymbol{E}(\boldsymbol{\mu}) \in \mathbb R^{N\times N}, \boldsymbol{B}(\boldsymbol{\mu}) \in \mathbb R^{N\times n_I}, \boldsymbol{C}(\boldsymbol{\mu}) \in \mathbb R^{n_o \times N}, \forall \boldsymbol{\mu} \in \mathcal P$, are the system matrices. The vector-valued $F: \mathbb R^N \times \mathcal P \mapsto \mathbb R^{N}$ is the nonlinear system operator and $\boldsymbol{u}(t, \boldsymbol{\mu}) \in \mathbb R^{n_I}$ is the vector of external inputs, which may also dependent on the parameter $\boldsymbol{\mu}$. The output response $\boldsymbol{y}(t, \boldsymbol{\mu}):=(y_1(t,\boldsymbol{\mu}), \ldots, y_{n_o}(t, \boldsymbol{\mu})) \in \mathbb R^{n_o}$ consists of the QoIs.
Such systems often arise from discretizing partial differential equations (PDEs) using numerical discretization schemes, or from some physical laws, for example, the modified nodal analysis (MNA) in circuit simulation. The number of DOFs $N$ is usually very large to reach high-resolution of the underlying physical process. Repeatedly solving the system in~(\ref{eq:fom}) at many samples of $\mu$ in a multi-query task is expensive. When $n_I>1$ and $n_o>1$, the system has multiple inputs and multiple outputs. Such problems are common in electrical or electromagnetic simulation~\cite{morCheFRetal23}. Our aim is to predict the output $\boldsymbol{y}(t, \boldsymbol{\mu})$ using SE-TFT, without knowledge of the system matrices and the expression of $F(\boldsymbol{x}(t, \boldsymbol{\mu}), \boldsymbol{\mu})$. In other words, we consider the system in~(\ref{eq:fom}) as a black box.

\subsection{Structure of TFT}
\label{sec:structure}
In this section, we introduce the general structure of TFT. Moreover, we adapt the input data and prediction value of TFT to the dynamical system in~\labelcref{eq:fom}. In particular, the covariates for TFT in~\cite{mlLimALetal21} are now referred to as the parameters $\mu_m, \, m=1,\ldots, n_p$. The main building blocks for TFT are the Gated Residual Network (GRN), the LSTM encoder-decoder and the temporal fusion decoder including a novel interpretable multi-head attention block. The overview of TFT is illustrated in Figure~\ref{fig:tft}. 
\begin{figure}
\centering
\includegraphics[width=140mm]{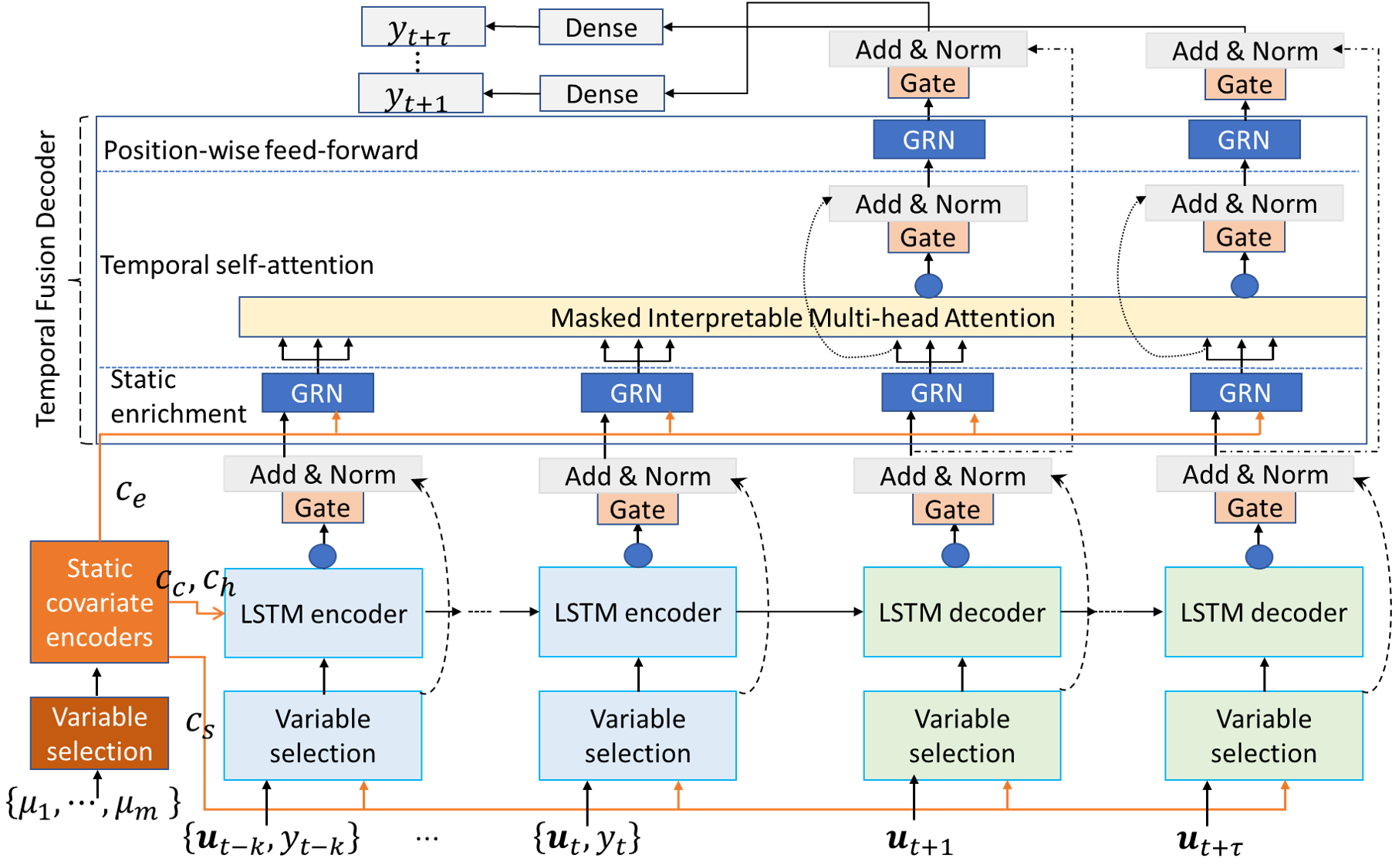}
\caption{The structure of TFT for parametric output prediction. The main structure is a copy of Fig.~2 in~\cite{mlLimALetal21}. Only the notation of the TFT input data, TFT prediction values and parameters (static metadata in~\cite{mlLimALetal21}) are different.}
\label{fig:tft}
\end{figure}
The inputs $\boldsymbol{u}_{t-k}, \ldots, \boldsymbol{u}_t$, and the outputs $y_{t-k},\ldots,y_t$, at the past time instances $t-k, \ldots, t$, and the input $\boldsymbol{u}_{t+1},\ldots, \boldsymbol{u}_{t+\tau}$, at future time instances are fed from the bottom into TFT. These variables then pass through a variable selection layer, a LSTM encoder-decoder network, a GRN layer, a multi-head attention layer, a position-wise feed-forward layer and a dense layer. The parameters $\mu_1, \ldots, \mu_p$, after being filtered by a variable selection network, are integrated by a static covariate encoder into different layers of TFT. Finally, TFT predicts the sequence of the output $y_{t+1}, \ldots, y_{t+\tau}$ at the future time instances and at the testing parameter samples. We briefly review each layer in TFT as below:
\begin{itemize}
\item The variable selection layer selects relevant input variables at each time step. It may also remove unnecessary noisy inputs which could negatively impact the performance of TFT. 
\item The LSTM is used to generate uniform temporal features $\boldsymbol{\phi}(t, -k), \ldots, \boldsymbol{\phi}(t, t+\tau)$ from the various input time series $\boldsymbol{u}_{t-k},\ldots, \boldsymbol{u}_t$, $y_{t-k},\ldots,y_t$, and $\boldsymbol{u}_{t+1},\ldots, \boldsymbol{u}_{t+\tau}$. 
The uniform temporal features then act as inputs of the temporal fusion decoder. 
\item The temporal fusion decoder consists of three blocks: the static enrichment block composed of GRNs, the multi-head attention layer with gating and the position-wise feed-forward layer that again is composed of GRNs. 
\item The gated multi-head attention contributes to the long-term prediction of TFT. The point-wise feed-forward network layer is an additional nonlinear processing of the outputs of the multi-head attention layer. 
\end{itemize} 
We refer to~\cite{mlLimALetal21} for more detailed and more exact explanations for each sub-network of TFT.

\section{Spatially-Enhanced Temporal Fusion Transformer}
\label{sec:MOTFT}
The original TFT predicts the time evolution of a scalar-valued function~(1D target). For predicting time sequences of multiple outputs $\boldsymbol{y}(t, \boldsymbol{\mu})$, TFT needs to be retrained upon change of each output. Here, we extend TFT to SE-TFT, which can predict multiple system outputs at once. The interpretable multi-head attention proposed in~\cite{mlLimALetal21} provides TFT with the ability of analyzing the temporal correlations between the features in a time sequence based on a single attention weight matrix. Rather than breaking the interpretability of TFT, SE-TFT extends its interpretability to multiple output cases, where the pairwise correlations between individual outputs, or in other words, the spatial correlations, at all time instances are also explored.

The structure of SE-TFT is introduced in two parts: reshaping the output data in the original TFT data set and a new masking scheme resulting in block-wise masked interpretable multi-head attention, which is a key part in SE-TFT. 

\subsection{Data preparation}
\label{sec:dataprep}
In the original TFT~\cite{mlLimALetal21}, the data set including the input signals and multiple outputs at $n_\mu$ parameter samples and $n_T$ time instances in the time interval $[0, T]$ is arranged as in~\labelcref{eq:tft_data}. The first column counts the number of data samples, the second column is the index number $i$ for the $i$-th parameter sample. The third column contains the time instances from $t_1$ to $t_{n_T}$ corresponding to each parameter sample. They are repeated $n_p$ times for the $n_p$ parameter samples. The $4$-th column to the $(n_I+3)$-th columns correspond to the samples of $n_I$ input signals at $n_T$ time instances and $n_p$ parameter samples. The next $p$ columns include the samples of $p$ parameters, each column corresponding to the samples of one parameter. Each sample is repeated for $n_T$ times, meaning that parameters remain fixed while the corresponding inputs and outputs evolve from $t_1$ to $t_{n_T}$. The last $n_o$ columns are the samples of $n_o$ outputs at $n_p$ samples of parameters and $n_T$ time instances. Data arrangement of these $n_o$ outputs is the main difference in the data preparation phase between the TFT and the SE-TFT. When training the TFT, columns corresponding to parameters, known inputs, outputs, time are detected and read into the training/validating/testing parts, where only one single column containing a single output can be imported and handled by the TFT.
\begin{equation}
\resizebox{\textwidth}{!}{$
\label{eq:tft_data}
\left(
\begin{array}{cccccccccc}
1 & 1 & t_1 & u_1(t_1,\boldsymbol{\mu}^1) & \ldots & u_{n_I}(t_1,\boldsymbol{\mu}^1) & \mu_1^1 \ldots \mu^1_p & y_1(t_1, \boldsymbol{\mu}^1) & \ldots & y_{n_o}(t_1, \boldsymbol{\mu}^1)\\
2 & 1 & t_2 & u_1(t_2, \boldsymbol{\mu}^1) & \ldots & u_{n_I}(t_2,\boldsymbol{\mu}^1) & \mu_1^1 \ldots \mu^1_p & y_1(t_2, \boldsymbol{\mu}^1) & \ldots & y_{n_o}(t_2, \boldsymbol{\mu}^1)\\
\vdots&\vdots & \vdots & \vdots& &\vdots& \vdots &\vdots & &\vdots \\
n_T & 1 & t_{n_T} & u_1(t_{n_T}, \boldsymbol{\mu}^1) & \ldots & u_{n_I}(t_{n_T},\boldsymbol{\mu}^1) & \mu_1^1 \ldots \mu^1_p & y_1(t_{n_T}, \boldsymbol{\mu}^1) & \ldots & y_{n_o}(t_{n_T}, \boldsymbol{\mu}^1)\\
n_T+1 & 2 & t_1 & u_1(t_1, \boldsymbol{\mu}^2)& \ldots & u_{n_I}(t_1,\boldsymbol{\mu}^2) & \mu_1^2  \ldots \mu^2_p & y_1(t_1, \boldsymbol{\mu}^2) & \ldots & y_{n_o}(t_1, \boldsymbol{\mu}^2)\\
n_T+2 & 2 & t_2 & u_1(t_2, \boldsymbol{\mu}^2) & \ldots & u_{n_I}(t_2,\boldsymbol{\mu}^2) &\mu_1^2 \ldots \mu^2_p & y_1(t_2, \boldsymbol{\mu}^2) & \ldots & y_{n_o}(t_2, \boldsymbol{\mu}^2)\\
\vdots&\vdots & \vdots &\vdots & &\vdots& \vdots &\vdots & &\vdots \\
2n_T & 2 & t_{n_T} & u_1(t_{n_T}, \boldsymbol{\mu}^2) & \ldots & u_{n_I}(t_{n_T},\boldsymbol{\mu}^2) & \mu_1^2 \ldots \mu^2_p & y_1(t_{n_T}, \boldsymbol{\mu}^2) & \ldots & y_{n_o}(t_{n_T}, \boldsymbol{\mu}^2)\\
\vdots&\vdots & \vdots &\vdots & &\vdots& \vdots & \vdots& &\vdots \\
(n_p-1)n_T+1& n_p& t_1& u_1(t_1, \boldsymbol{\mu}^{n_p}) & \ldots & u_{n_I}(t_1,\boldsymbol{\mu}^{n_p}) & \mu_1^{n_\mu} \ldots \mu^{n_\mu}_p& y_1(t_1, \boldsymbol{\mu}^{n_p})& \ldots& y_{n_o}(t_1, \boldsymbol{\mu}^{n_p}) \\
\vdots&\vdots & \vdots &\vdots & &\vdots& \vdots &\vdots & &\vdots \\
n_pn_T&n_p & t_{n_T}& u_1(t_{n_T}, \boldsymbol{\mu}^{n_p})& \ldots & u_{n_I}(t_{n_T},\boldsymbol{\mu}^{n_p}) & \mu_1^{n_\mu} \ldots \mu^{n_\mu}_p& y_1(t_{n_T}, \boldsymbol{\mu}^{n_p})& \ldots& y_{n_o}(t_{n_T}, \boldsymbol{\mu}^{n_p})
\end{array}\right),
$}
\end{equation}
where $\boldsymbol{\mu}^m:=(\mu_1^m, \ldots, \mu_p^m)^T, m=1,\ldots, n_p$. 

During the data preparation phase for SE-TFT, the columns of the multiple outputs in~\labelcref{eq:tft_data} are stacked into a single output column as shown in~\labelcref{eq:motft_data}. Because of the new alignment of the output column, each row of the matrix block on the left side of the first output column in~\labelcref{eq:tft_data} is duplicated $n_o$ times, resulting in a reshaped dataset in~\labelcref{eq:motft_data} with $n_p n_T n_o$ rows. From~\labelcref{eq:motft_data}, we see that every $n_o$ rows are data samples for $n_o$ different outputs at the same time instance. The new alignment mixes the dimension of time and the dimension of the output (spatial dimension), and leads to a spatial-temporal sequence rather than the temporal-only sequence in~\labelcref{eq:tft_data}. After SE-TFT is trained with this spatial-temporal data, a solution sequence containing all outputs at future time instances can be predicted in one step.
\begin{equation}
\label{eq:motft_data}
\left(
\begin{array}{cccccccc}
1 & 1 & t_1 & u_1(t_1,\boldsymbol{\mu}^1) & \ldots & u_{n_I}(t_1,\boldsymbol{\mu}^1) & \mu_1^1 \ldots \mu^1_p & y_1(t_1, \boldsymbol{\mu}^1) \\
\vdots&\vdots & \vdots & \vdots& &\vdots&\vdots & \vdots\\
n_o & 1 & t_1 & u_1(t_1,\boldsymbol{\mu}^1) & \ldots & u_{n_I}(t_1,\boldsymbol{\mu}^1)& \mu_1^1 \ldots \mu^1_p & y_{n_o}(t_1, \boldsymbol{\mu}^1)\\
n_o+1 & 1 & t_2 & u_1(t_2, \boldsymbol{\mu}^1)& \ldots & u_{n_I}(t_2,\boldsymbol{\mu}^1) & \mu_1^1 \ldots \mu^1_p & y_1(t_2, \boldsymbol{\mu}^1)\\
\vdots&\vdots & \vdots & \vdots& &\vdots&\vdots & \vdots \\
n_T n_o & 1 & t_{n_T} & u_1(t_{n_T}, \boldsymbol{\mu}^1)& \ldots & u_{n_I}(t_{n_T},\boldsymbol{\mu}^1) & \mu_1^1  \ldots  \mu^1_p & y_{n_o}(t_{n_T}, \boldsymbol{\mu}^1)\\
n_T n_o+1 & 2 & t_1 & u_1(t_1, \boldsymbol{\mu}^2)& \ldots & u_{n_I}(t_1,\boldsymbol{\mu}^2) & \mu_1^2  \ldots  \mu^2_p & y_1(t_1, \boldsymbol{\mu}^2)\\
\vdots&\vdots & \vdots & \vdots& & \vdots & \vdots & \vdots\\
n_p n_T n_o& n_p & t_{n_T}& u_1(t_{n_T}, \boldsymbol{\mu}^{n_p})& \ldots & u_{n_I}(t_{n_T},\boldsymbol{\mu}^{n_p})&  \mu_1^{n_p} \ldots  \mu^{n_p}_p& y_{n_o}(t_{n_T}, \boldsymbol{\mu}^{n_p})
\end{array}
\right)
\end{equation}

\subsection{Block-wise masked interpretable multi-head attention}

The masked interpretable multi-head attention enables the interpretability of TFT in~\Cref{fig:tft}. In the following, we briefly explain the self- attention~\cite{mlVasSP17} used in TFT and the further proposed masked interpretable multi-head attention~\cite{mlLimALetal21}. Then we propose the block-wise masked interpretable multi-head attention for SE-TFT.

Introduced in~\cite{mlVasSP17}, the self-attention mechanism is a key element in the architecture of any transformer to capture long-term correlations between the features in an input time sequence. \Cref{fig:self-att} illustrates an example of the masked self-attention mechanism, where the number $M$ of features in the time sequence is $M=5$. $M$ is also the length of the time sequence $\{y_{t-k},~\ldots,~ y_{t+\tau} \}$ in the TFT, denoted as $M = n_t$.

\begin{figure}
\centering
\includegraphics[scale=0.5]{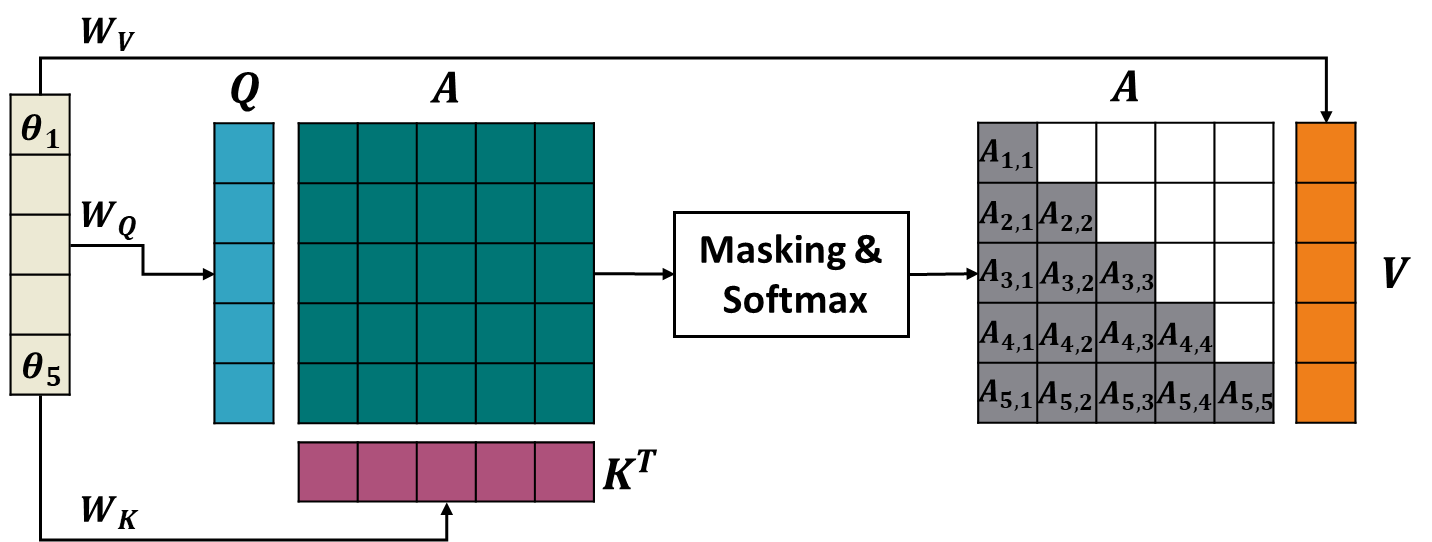}
\caption{Structure of the masked self-attention mechanism proposed in~\cite{mlVasSP17} and used in the TFT.}
\label{fig:self-att}
\end{figure}

The self-attention mechanism is used in TFT, where the matrix of inputs $\boldsymbol{\Theta}= [\boldsymbol{\theta}_1, \ldots, \boldsymbol{\theta}_{M}] \in \mathbb{R}^{M \times d_{model}}$ with $\boldsymbol{\theta}_i \in \mathbb R^{d_{model}}$ is converted to $\boldsymbol{Q} \in \mathbb{R}^{M \times d_{k}}$~(Queries), $\boldsymbol{K} \in \mathbb{R}^{M \times d_{k}}$~(Keys) and $\boldsymbol{V} \in \mathbb{R}^{M \times d_{v}}$~(Values) via linear transformations, i.e., $\boldsymbol{Q} = \boldsymbol{\Theta} \boldsymbol{W_{Q}}$, $\boldsymbol{K} = \boldsymbol{\Theta} \boldsymbol{W_{K}}$ and $\boldsymbol{V} = \boldsymbol{\Theta} \boldsymbol{W_{V}}$, where $ \boldsymbol{W_{Q}}$, $\boldsymbol{W_{K}} \in \mathbb{R}^{d_{model} \times d_{k}}$ and $ \boldsymbol{W_{V}} \in \mathbb{R}^{d_{model} \times d_{v}}$. The self-attention function is expressed as 
\begin{equation}
\textit{Attention}(\boldsymbol{Q}, \boldsymbol{K}, \boldsymbol{V}) = \boldsymbol{A}(\boldsymbol{Q},\boldsymbol{K}) \boldsymbol{V}.
\end{equation}

In the original TFT, the information of $\boldsymbol{\mu}$, $\boldsymbol{u}(t, \boldsymbol{\mu})$ and the scalar-valued output $y(t, \boldsymbol{\mu})$ at each time step is integrated into each feature $\boldsymbol{\theta}_i, i= 1,\ldots, M$ in $\boldsymbol{\Theta}$, via embedding, variable selection and local processing with LSTM. $\boldsymbol{A} \in \mathbb{R}^{M \times M}$ is the attention weight matrix computed by a scaled dot-product via $\boldsymbol{A}(\boldsymbol{Q},\boldsymbol{K}) = \textit{Softmax}\,(\boldsymbol{Q} \boldsymbol{K}^T / \sqrt{d_k})$. $d_{model}$ is the dimension of hidden states defined across the TFT model. After linear transformation, the input dimension is converted to $d_k$ and $d_v$ for the query/key sequence and for the value sequence, respectively. The magnitude of the entry $\boldsymbol{A}_{i,j}$ in the attention weight matrix $\boldsymbol{A}$ interprets the correlation between the feature at $t_i$ and the feature at $t_j$ in the time series. Masking prevents the transformer from obtaining the ``future" information, i.e., all entries $\boldsymbol{A}_{i,j}, \, i < j$ are masked and correspond to the empty entries in $\boldsymbol{A}$ in \Cref{fig:self-att}, meaning that the future feature at $t_j$ has no influence on the past feature at $t_i, \, i < j$.

In the framework of the multi-head attention from~\cite{mlVasSP17}, self-attention is employed $n_h$ times in parallel resulting in $n_h$ heads with $n_h$ attention weight matrices $\boldsymbol{A}_h, h=1, \ldots, n_h$. However, various $\boldsymbol{A}_h$ are not informative enough to describe the correlation between the features in a single time sequence. The TFT in~\cite{mlLimALetal21} provides the interpretability of the multi-head attention by averaging the different attention weight matrices $\boldsymbol{A}_h, \, h=1, \ldots, n_h$ to a single attention weight matrix $\overline{\boldsymbol{A}}$. Interpretable multi-head attention resembles the formulation of the self-attention, allowing simple interpretability studies by analysing a single averaged attention weight matrix $\overline{\boldsymbol{A}}$, like $\boldsymbol{A}$ in the self-attention. The averaged attention weight matrix $\overline{\boldsymbol{A}}$ can be computed via:
\begin{equation}
\overline{\boldsymbol{A}} = \frac{1}{n_h} \sum_{h=1}^{n_h} \boldsymbol{A}_h  = \frac{1}{n_h} \sum_{h=1}^{n_h} \textit{Softmax}\,(\boldsymbol{Q}_{h} \boldsymbol{K}_{h}^T / \sqrt{d_k}),
\end{equation}
where $\boldsymbol{Q}_{h}$ and $\boldsymbol{K}_{h}$ are queries and keys in each head. 

In SE-TFT, the spatial dimension of the output is merged into the dimension of the time. Unlike the original TFT, $M$ equals to $n_{o} \times n_t$ in the spatial (output locations)-temporal sequence. To maintain the interpretability of the attention weight matrix, the mask must be added in a block-wise manner. As illustrated in the right part of~\Cref{fig:stair}, the length of each time step in SE-TFT corresponds to the number of outputs, $n_o$. For clarity, we use the simple case $n_t = n_o = 3$ as an example. Instead of each entry of $\widetilde{\boldsymbol{A}}$, each bock $\widetilde{\boldsymbol{A}}_{i,j}, \, i<j$, is masked, the proposed block-wise masked attention weight matrix $\widetilde{\boldsymbol{A}}$ shows an extended interpretability. No masking is applied within each unmasked block to ensure that pair of features corresponding to different outputs can interact with each other. In particular, the entry $a_{k,l}, \, k,l = 1, \ldots, n_o$, in $\widetilde{\boldsymbol{A}}_{i,j}$ provides the correlation between the feature related to the $k$-th output at $t_i$ and the feature related to the $l$-th output at $t_j$.
\begin{figure}
\centering
\includegraphics[scale=0.4]{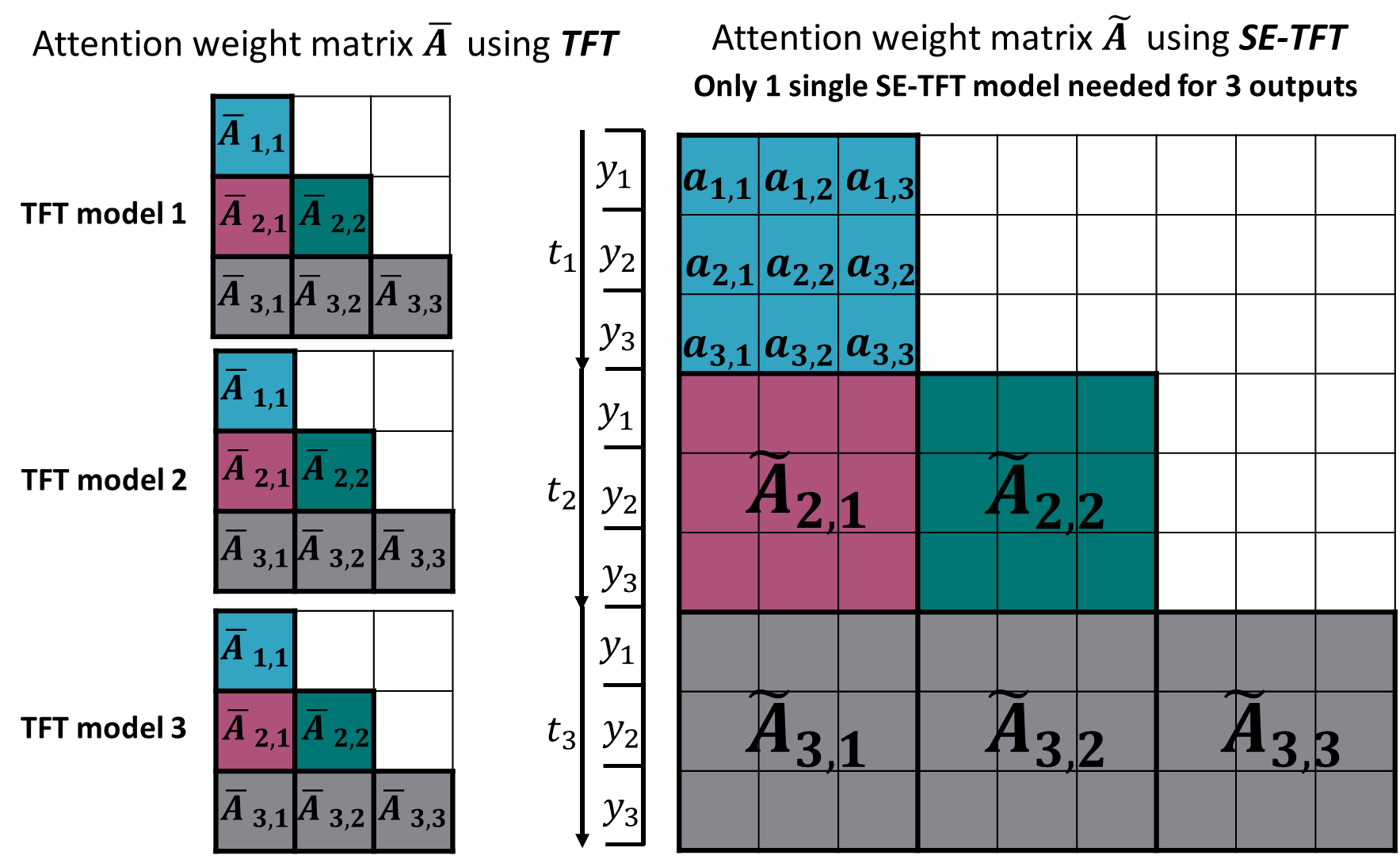}
\caption{The structure of the block-wise masked attention weight matrix $\widetilde{\boldsymbol{A}}$ in a single SE-TFT~(right) compared to the normal masked attention weight matrix $\overline{\boldsymbol{A}}$~(left) in three separate TFT models for three outputs. Here, we use $n_t = 3$ as an example.}
\label{fig:stair}
\end{figure}

\subsection{Other implementation details}

After reading the formatted row data from~\labelcref{eq:motft_data} in the form of a data file, e.g., a CSV file, SE-TFT splits the data into training, validation, and testing sets based on the parameter ID, as shown in the second column in~\labelcref{eq:motft_data}. More specifically, the raw data, containing $n_p$ groups, are divided into $n_{p_{train}}$, $n_{p_{validate}}$ and $n_{p_{test}}$ for training, validation, and testing, respectively. The past time interval $[t-k, t]$ and the forecast time interval $[t, t+\tau]$ can be set by specifying $n_k$ and $n_{\tau}$ time instances depending on the application cases and user's computational resources. When the length of the time sequence $n_t = n_k + n_{\tau}$ in SE-TFT is smaller than the total number of time steps collected in the data file ($n_t < n_T$), multiple subsets can be extracted from the whole-time sequence. For example, if time sequences corresponding to parameter sample $\boldsymbol{\mu}_i \, (i= 1, \ldots, n_p)$ contain time steps covering the interval $[0, 1000]$ with $n_T = 1000$, and the past and the forecast time interval in SE-TFT are set as $[t-1, t]$ and $[t, t+100]$, respectively (i.e., $n_t = 101$), multiple subsets can be extracted out of the whole-time sequence of $n_T=1000$ time steps, each consisting of 101 time instances. The number of subsets $n_{\Omega}$ corresponding to each of the $n_p$ parameter samples can be determined by the user, leading to $n_{\Omega} n_p$ subsets in total. SE-TFT is trained by sweeping over all the subsets. When constructing SE-TFT, the building blocks shown in~\Cref{fig:tft} are assembled, and the Adam optimizer is employed for optimizing the weights $\boldsymbol{W}$ of the SE-TFT. We employ two types of loss functions to train SE-TFT and evaluate its performance of predicting different types of outputs. The first loss function, designed to support output value predictions, is defined in $L_1$-norm, i.e.,
\begin{equation}
\label{eq:maeloss}
\mathcal{L}_{MAE}(\Omega, \boldsymbol{W}) = \sum_{\boldsymbol{y} \in \Omega} \sum_{i=n_k}^{n_t} \frac{\|\boldsymbol{y}(t_i) - \tilde{\boldsymbol{y}}(t_i)\|_1}{n_{\Omega} n_{p_{train}} n_{\tau}},
\end{equation}
where $\Omega$ is the set of training data containing $n_{p_{train}}$ parameter samples sets. $\boldsymbol{W}$ includes the trainable weights of TFT. The following theorem can be easily proved.
\begin{theorem}
The $L_1$-norm loss in~\labelcref{eq:maeloss} is equivalent to twice the quantile loss with quantile value $q=0.5$ used in~\cite{mlLimALetal21}. 
\end{theorem}
\begin{proof}
In fact, the quantile loss in~\cite{mlLimALetal21} is defined (for $n_o=1$) as,
\begin{equation}
\label{eq:qloss}
\mathcal{L}_{q}(\Omega, \boldsymbol{W}) = \sum_{\boldsymbol{y} \in \Omega} \sum_{i=n_k}^{n_t} \sum_{j = 1}^{n_o} \frac{\mathcal{Q}(y_j(t_i),~\tilde{y}_j(t_i),~q)}{n_{\Omega} n_{p_{train}} n_{\tau} n_o}.
\end{equation}
In~\labelcref{eq:qloss}, $\mathcal{Q}(y,~\tilde{y},~q)$ is formed as:
\begin{equation}
\label{eq:quantile_loss}
\mathcal{Q}(y,~\tilde{y},~q) = q (y - \tilde{y})_+ + (1-q) (\tilde{y} - y)_+,
\end{equation}
where $(\cdot)_+ = \text{max}(0, \cdot)$. Note that when $q = 0.5$, $\mathcal{Q}$ can be rewritten as $0.5(y - \tilde{y})_+ + 0.5 (\tilde{y} - y)_+$. Using the definition of $(\cdot)_+$, we obtain
\begin{equation}
  \mathcal{Q}(y,~\tilde{y},~q=0.5) =
    \begin{cases}
    0.5(y - \tilde{y}) & \text{if } \tilde{y} \leq y \\
    -0.5(y - \tilde{y}) & \text{if } \tilde{y} > y.
    \end{cases}
\end{equation}
As a result, $\mathcal{Q} = 0.5|y - \tilde{y}|$, so that $\sum\limits_{j=1}^{n_o} |y_j-\tilde y_j|= \|\boldsymbol{y} - \tilde{\boldsymbol{y}}\|_1$. Finally, 
\begin{equation}
\mathcal{L}_{q=0.5}(\Omega, \boldsymbol{W}) = 0.5 \sum_{\boldsymbol{y} \in \Omega} \sum_{i=n_k}^{n_t} \frac{\|\boldsymbol{y}(t_i) - \tilde{\boldsymbol{y}}(t_i)\|_1}{n_{\Omega} n_{p_{train}} n_{\tau}} = 0.5 \mathcal{L}_{MAE}.
\end{equation}
\end{proof}

The $L_1$-norm loss in~\labelcref{eq:maeloss} is usually denoted as the mean absolute error (MAE) loss function. 

The second loss function, $\mathcal{L}_{q}$, is used for quantile forecasting and is identical to that employed in the original TFT model. This quantile loss, defined in \cref{eq:qloss} and \cref{eq:quantile_loss}, is applied with three output quantiles ($q = 0.1, 0.5, 0.9$) in the numerical experiments. Under these settings, SE-TFT simultaneously predicts the 10th, 50th, and 90th percentiles over the forecast horizon, thereby providing prediction intervals that capture the range of the outputs at each prediction step. In the next section, we apply the first loss function in~\cref{eq:maeloss} to achieve output value predictions for the Lorenz–63 model and the coupled electrochemical kinetics and diffusion model. The second loss function is applied to the FitzHugh–Nagumo model for quantile forecasting of the outputs.

\section{Numerical results of SE-TFT}
\label{sec:pred_results}

This section presents the performance of SE-TFT on three dynamical systems from different engineering applications. We use the mean relative $\mathcal{L}_2$ error, denoted by $\epsilon_k$, to evaluate the forecasting accuracy. The output $\boldsymbol{y}_k^j$ is a vector collecting the k-th output at all future time instances $\{t_{n_k}, t_{n_k+1}, \ldots, t_{n_t}\}$ corresponding to the $j$-th testing parameter. The mean relative error is defined as:
\begin{equation}
\label{eq:err_ind}
\epsilon_k = \frac{1}{N}\sum_{j=1}^N\frac{ \| \hat{\boldsymbol{y}}_{k}^j - \boldsymbol{y}_{k}^j \|_2}{\| \boldsymbol{y}^j_{k} \|_2},
\end{equation}
where $\boldsymbol{y}^j_k = 
\bigl[
y_k(t_{n_k}, \boldsymbol{\mu}_j),\;
y_k(t_{n_k+1}, \boldsymbol{\mu}_j),\;
\ldots,\;
y_k(t_{n_t}, \boldsymbol{\mu}_j)
\bigr]$.

All numerical examples were conducted on a personal computer equipped with a 12th Gen Intel\textsuperscript{\textregistered} Core\textsuperscript{\texttrademark} i5-12600K CPU, 31 GB of RAM, a 64-bit operating system, and an NVIDIA\textsuperscript{\textregistered} RTX\textsuperscript{\texttrademark} A4000 GPU. The code was implemented in Python 3.10.8 using TensorFlow 2.11.0.

\subsection{Lorenz-63 model}
\label{sec:Lorenz}

The Lorenz-63 model is a simplified mathematical model to describe chaotic dynamics, which is defined by the following system of ODEs: 
\begin{equation}
\label{eq:Lorenz}
\frac{dy_1}{dt} = \sigma \left(y_2 - y_1 \right), \quad \frac{dy_2}{dt} = y_1 \left(\rho - y_3 \right) - y_2, \quad \frac{dy_3}{dt} = y_1 y_2 - \beta y_3,
\end{equation}
where $\sigma = 10$, $\rho = 28$, and $\beta = 8/3$. The Lorenz-63 model is parametrized by varying the initial conditions. This model is often used as the academic benchmark for proof-of-concepts studies of dynamical systems learning~\cite{Lorenz1963, morBruPK16, morChaLKetal19, mlGenZ22}. We initialized the system with random initial states $(y_1(t_1), y_2(t_1), y_3(t_1))$ in uniform distributions: $y_1(t_1) \sim U(-20, 20)$, $y_2(t_1) \sim U(-20, 20)$, and $y_3(t_1) \sim U(10, 40)$. The numerical integration of ODEs was performed using a 4th order Runge-Kutta integrator with a time step of $0.01$. The QoIs of the example are simply the three states, i.e., $\boldsymbol{y}=(y_1, y_2, y_3)^T \in \mathbb R^3$.
The training data, validating data and testing data correspond to time series with $n_{p_{train}} = 2048$, $n_{p_{validate}} = 64$ and $n_{p_{test}} = 256$ groups of random initial states, respectively. Each training time sequence consists of $256$ time steps. Each validating and each testing time sequence contains $1024$ time steps.

In the training phase, both the training and validating data corresponding to each given parameter sample are chunked into partially overlapping 8 subsets within the total time series, each consisting of $n_t = 129$ time steps out of the $n_T = 256$ training time steps. SE-TFT is repeatedly trained for each subset, where the outputs $\boldsymbol{y}$ at the first $9$ time steps are set as the observed outputs at the past time instances, while the outputs at the subsequent $120$ time steps are to be predicted by the SE-TFT. Number of epochs of training SE-TFT is set as 5000, while other hyperparameters are shown in \Cref{tab:hyper_63}. In the testing phase, SE-TFT predicts $\boldsymbol{y}(t)$ at the the next $120$ time steps corresponding to all testing initial states in the testing set $\{\boldsymbol{y}^*(t_1)\}$, in a single operation. Within $n_{p_{test}}=256$ groups of testing data, we pick the first $129$ time steps and the last $129$ time steps ($n_t = 129$) out of the $n_T = 1024$ testing time steps, leading to 2 subsets of time sequences in each group. Consequently, we have 512 testing cases in total. For each testing case, the vector $\boldsymbol{y}^*(t_1)$ at its first time instance is considered as initial condition and $\boldsymbol{y}^*$ at the first $9$ time instances are the observed outputs. The outputs at the subsequent $120$ time steps are predicted. 

SE-TFT also supports longer-term prediction via autoregressive rollouts, in which its own predictions are recursively fed into the look-back window until the specified time horizon is reached. To assess its autoregressive performance, the same trained model is evaluated on an additional task involving multi-output prediction over 240 time steps. Specifically, after predicting the first 120 time steps, the predicted states from the most recent $9$ time instances are treated as newly observed outputs and fed back into the model. SE-TFT then predicts the subsequent 120 time steps, yielding a total forecast horizon of 240 time steps obtained through two autoregressive steps. Note that the reference data in this task is collected at the first 249 time steps and the last 249 time steps ($n_t = 249$) from the full set of $n_T = 1024$ testing time steps.
\begin{table}
\small
\caption{Lorenz-63 model: the hyperparameters for training SE-TFT.}
\label{tab:hyper_63}
  \begin{tabular}{|c|c|c|c|c|c|} \hline
   Learning rate & Dropout rate & Number of heads & $d_{model}$ & Minibatch size & Max gradient norm \\ \hline
   0.001 & 0.2 & 4 & 160 & 256 & 1.0 \\ \hline
\end{tabular}
\end{table}

Six randomly-picked testing cases for 120 time steps are illustrated in~\Cref{fig:lorenz_sol}, which show different complex and chaotic trajectories that evolve from different initial conditions. \Cref{fig:lorenz_sol_a} presents the autoregressive predictions based on the same trained model for 240 time steps. The errors of these two tasks are shown in~\Cref{tab:error_lorenz}. It can be observed that that SE‑TFT achieves an error of approximately 
$5\%$ for a single autoregressive step. The error increases slightly as the forecast horizon becomes longer.
\begin{table}
\caption{Lorenz-63 model: the value of the error in~\labelcref{eq:err_ind} for two tasks over all testing cases.}
\centering
\begin{tabular}{llll}
\toprule
Numerical tasks & $\epsilon_{y_{1}}(\boldsymbol{\mu}^*_k)$ & $\epsilon_{y_{2}}(\boldsymbol{\mu}^*_k)$ & $\epsilon_{y_{3}}(\boldsymbol{\mu}^*_k)$ \\
\midrule
Lorenz-63 & 0.0431 & 0.0536 & 0.0186 \\
Lorenz-63 (autoregressive) & 0.0932 & 0.1136 & 0.0418 \\
\bottomrule
\end{tabular}
\label{tab:error_lorenz}
\end{table}
\begin{figure}
\begin{subfigure}{0.33\textwidth}
\includegraphics[width=0.99\linewidth]{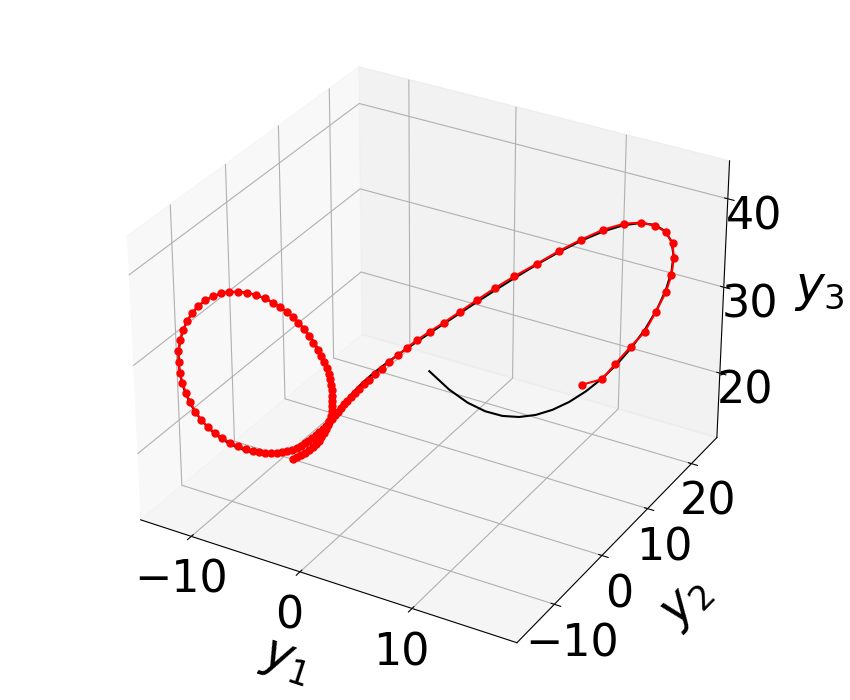}
\label{fig:Lo-65}
\end{subfigure}
\begin{subfigure}{0.33\textwidth}
\includegraphics[width=0.99\linewidth]{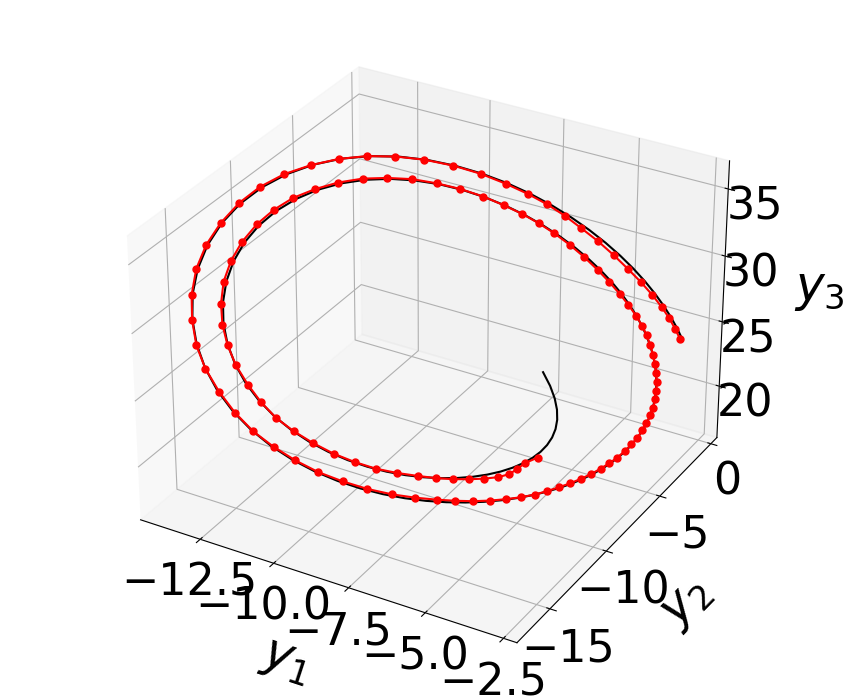}
\label{fig:Lo-129}
\end{subfigure}
\begin{subfigure}{0.33\textwidth}
\includegraphics[width=0.99\linewidth]{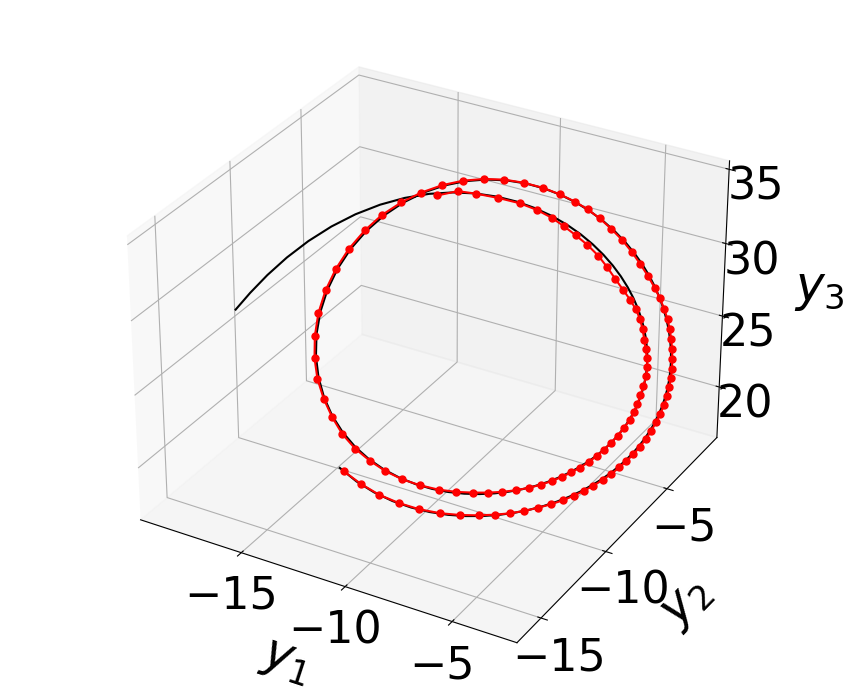}
\label{fig:Lo-193}
\end{subfigure}
\begin{subfigure}{0.33\textwidth}
\includegraphics[width=0.99\linewidth]{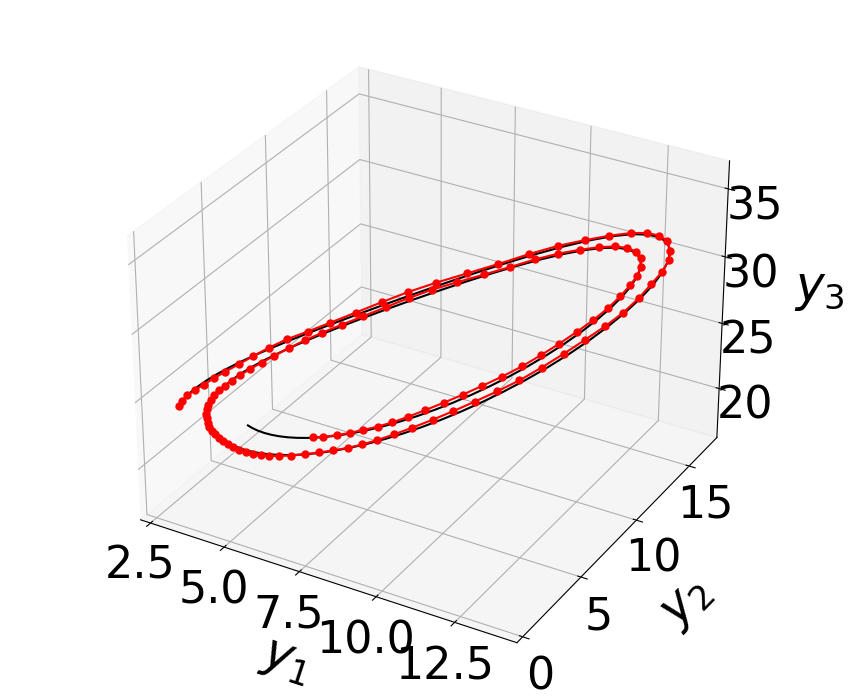} 
\label{fig:Lo-2}
\end{subfigure}
\begin{subfigure}{0.33\textwidth}
\includegraphics[width=0.99\linewidth]{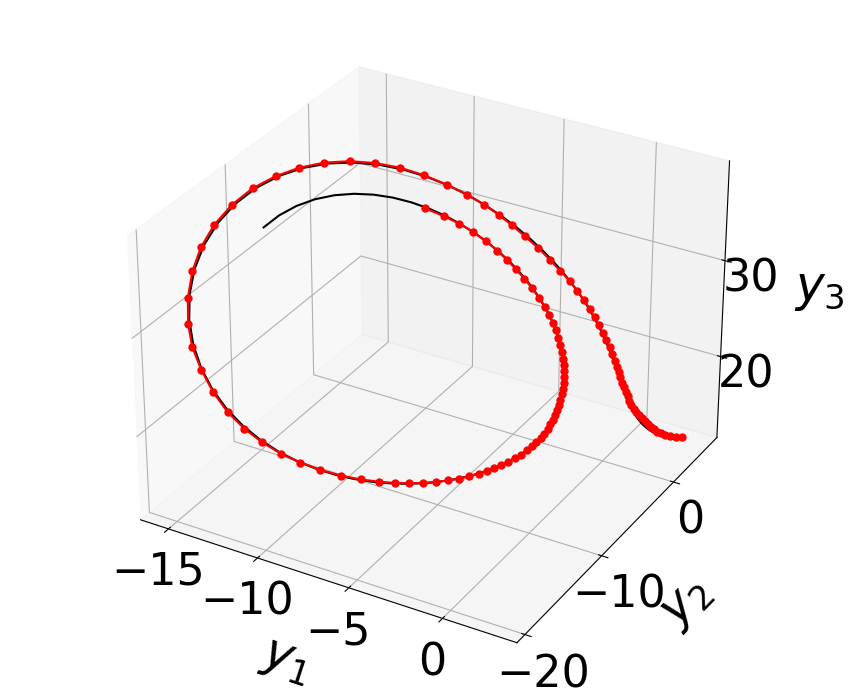}
\label{fig:Lo-66}
\end{subfigure}
\begin{subfigure}{0.33\textwidth}
\includegraphics[width=0.99\linewidth]{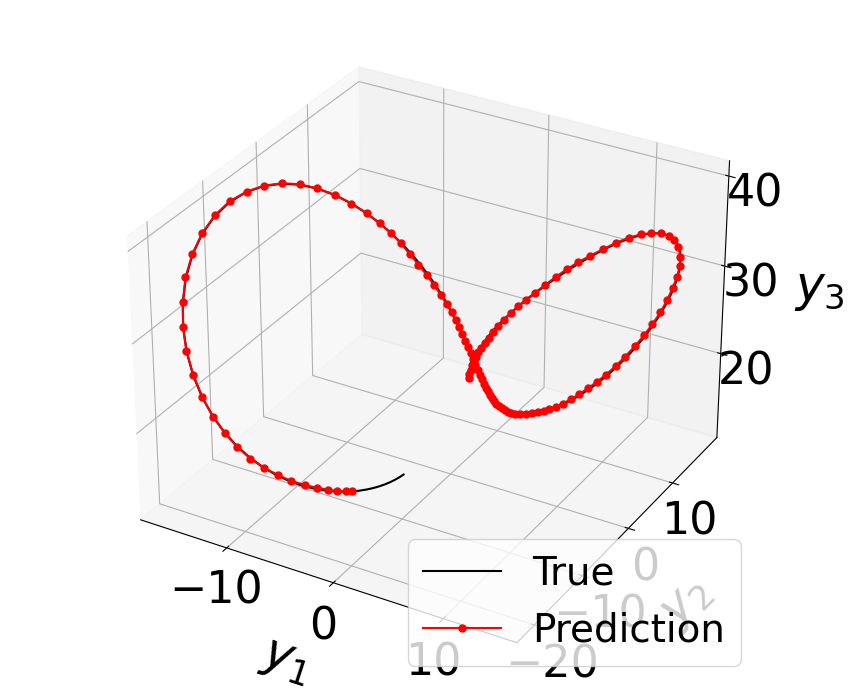}
\label{fig:Lo-450}
\end{subfigure}
\caption{Lorenz-63 model: the predicted solution and the reference solution.}
\label{fig:lorenz_sol}
\end{figure}
\begin{figure}
\begin{subfigure}{0.33\textwidth}
\includegraphics[width=0.99\linewidth]{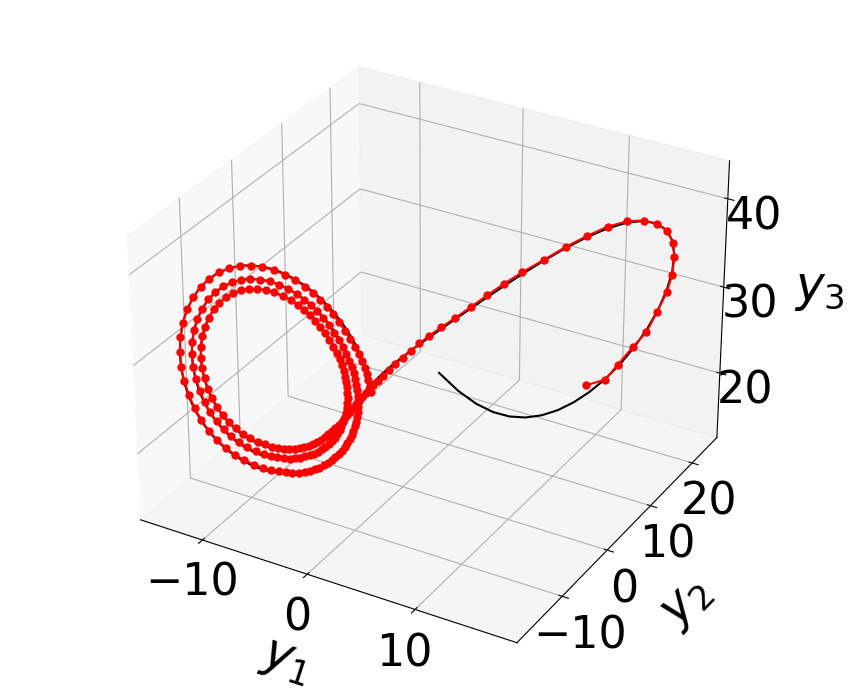}
\label{fig:Lo-65-q}
\end{subfigure}
\begin{subfigure}{0.33\textwidth}
\includegraphics[width=0.99\linewidth]{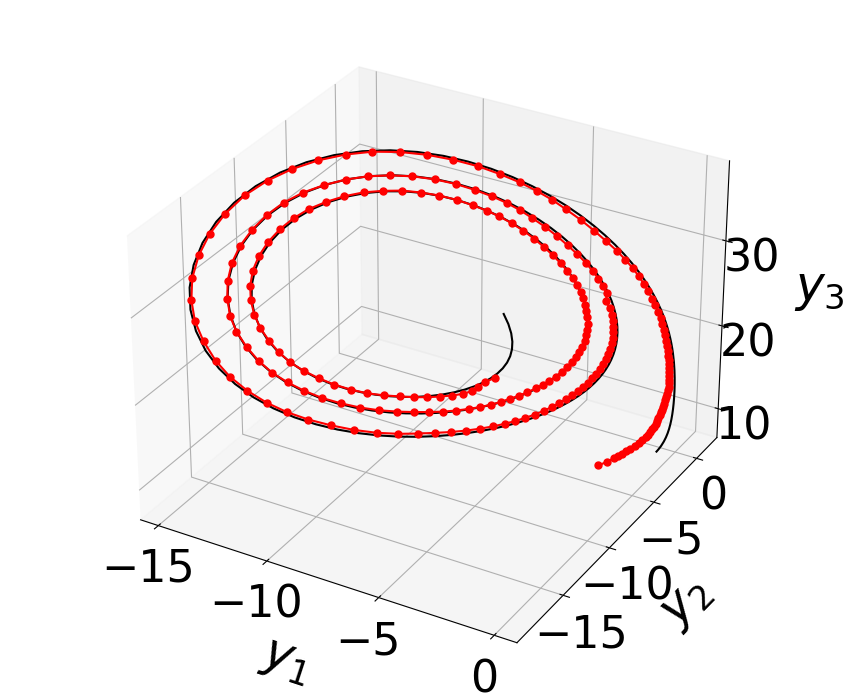}
\label{fig:Lo-129-q}
\end{subfigure}
\begin{subfigure}{0.33\textwidth}
\includegraphics[width=0.99\linewidth]{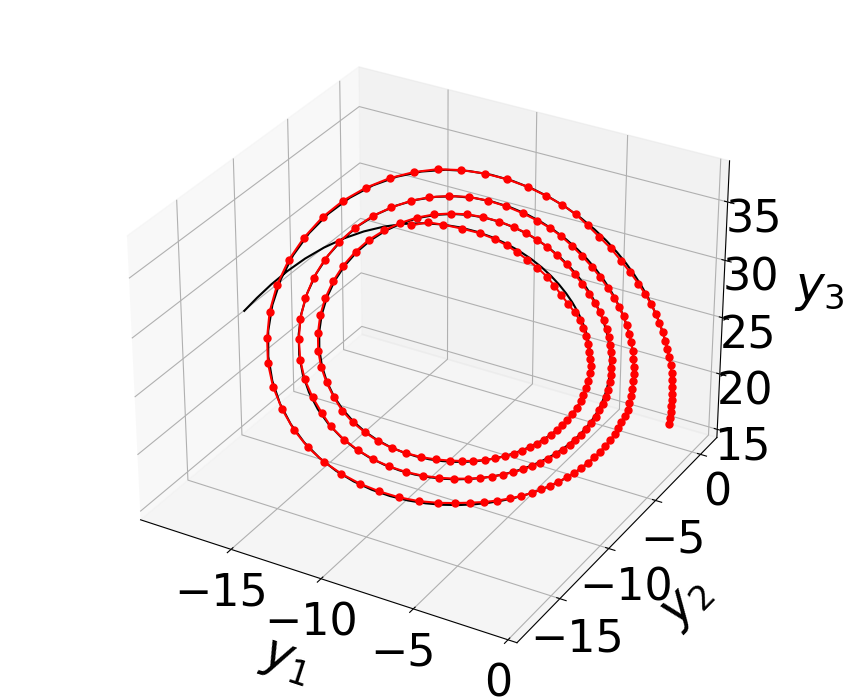}
\label{fig:Lo-193-q}
\end{subfigure}
\begin{subfigure}{0.33\textwidth}
\includegraphics[width=0.99\linewidth]{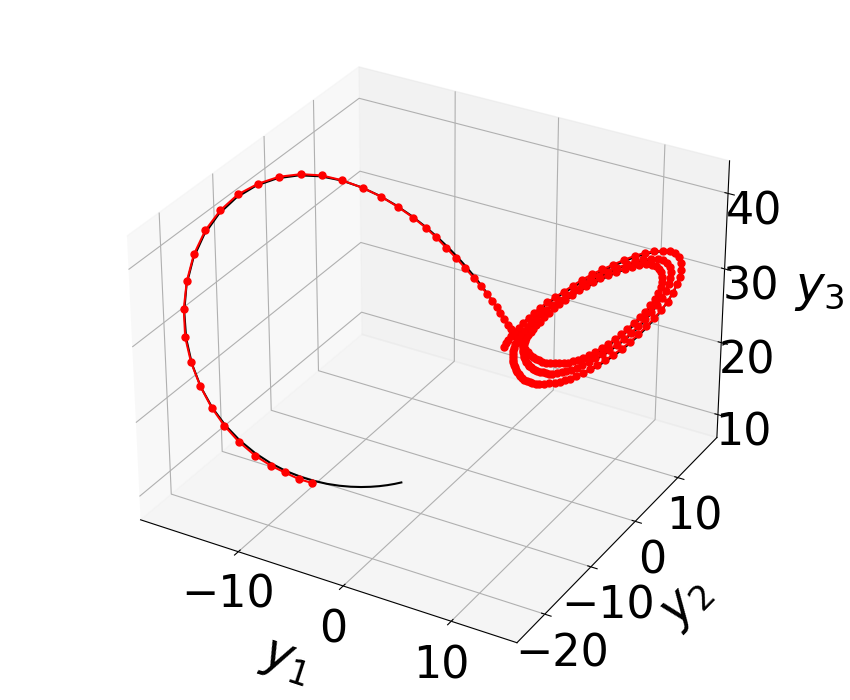}
\label{fig:Lo-2-q}
\end{subfigure}
\begin{subfigure}{0.33\textwidth}
\includegraphics[width=0.99\linewidth]{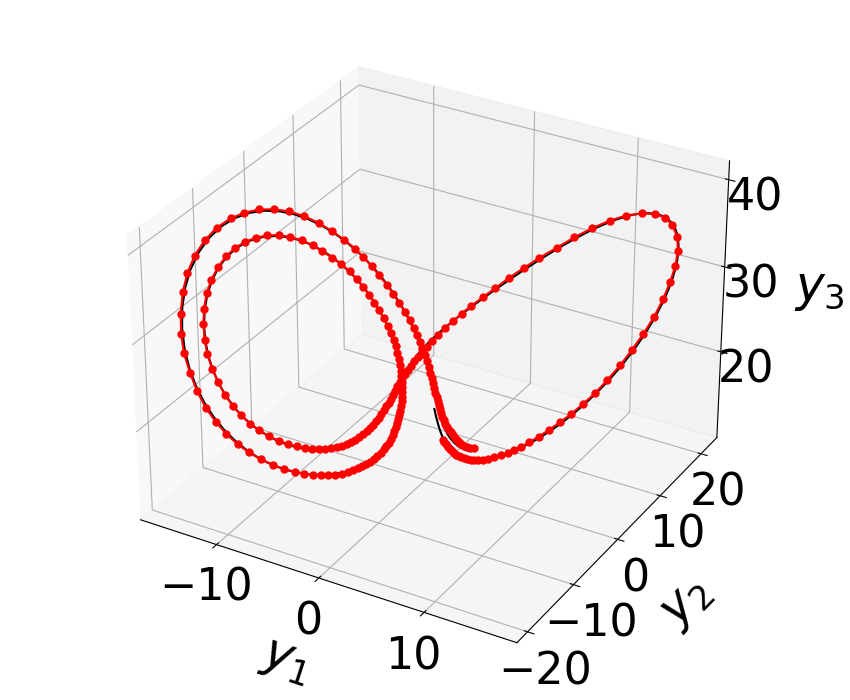}
\label{fig:Lo-66-q}
\end{subfigure}
\begin{subfigure}{0.33\textwidth}
\includegraphics[width=0.99\linewidth]{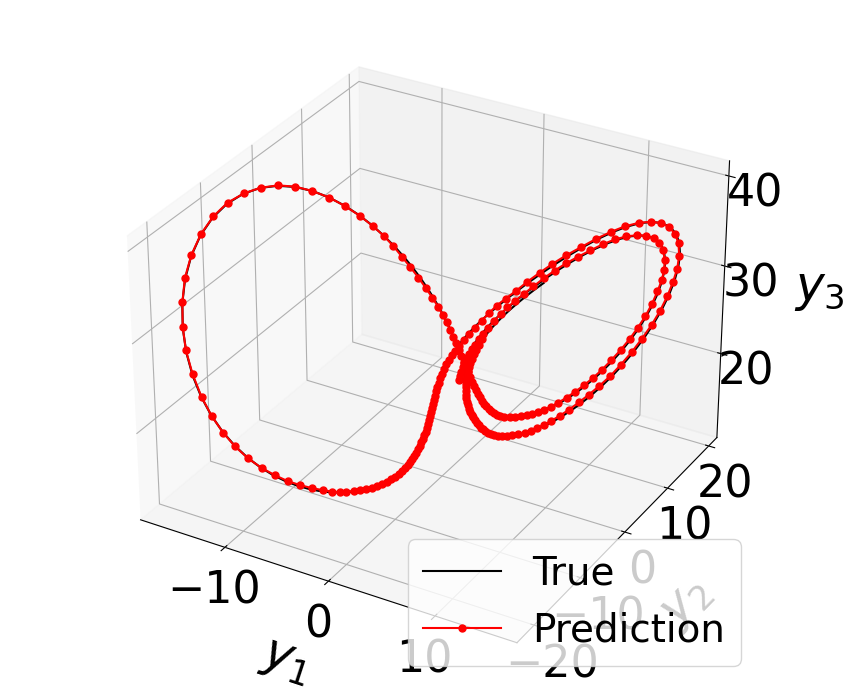}
\label{fig:Lo-450-q}
\end{subfigure}
\caption{Lorenz-63 model: the predicted solution~(autoregressive) and the reference solution.}
\label{fig:lorenz_sol_a}
\end{figure}

\subsection{The FitzHugh-Nagumo model}
\label{sec:FitzHugh}
The FitzHugh-Nagumo model is used to exam the response of neurons to external stimuli~\cite{FitzHug}. When the external stimulus exceeds a certain threshold value, the system will exhibit a characteristic excursion in phase space, representing activation and deactivation of the neuron. The model~\cite{morwiki_modFHN} is described by the coupled PDE-ODE system,

\begin{align}\label{eq:FitzHugh}
\varepsilon \frac{\partial v(x,t, \varepsilon, c)}{\partial t}&=\varepsilon \frac{\partial^2 v(x,t,\varepsilon, c)}{\partial x^2}+f(v(x,t, \varepsilon, c))-w(x,t,\varepsilon, c)+c, \nonumber \\
\frac{\partial w(x,t,\varepsilon, c)}{\partial t}&=bv(x,t,\varepsilon, c)-\gamma w(x,t,\varepsilon, c)+c,
\end{align}
with the nonlinear function $f(v)=v(v-0.1)(1-v)$, $b=0.5, \gamma=2$. The two independent parameters are $\varepsilon \in [0.01, 0.04]$ and $c \in [0.025, 0.075]$, so that $\boldsymbol{\mu}=(\varepsilon, c)^T$. The initial and boundary conditions are 
\begin{align}
\label{eq:FitzHughIB}
&v(x,0,\varepsilon, c)=0, &w(x,0,\varepsilon, c)=0, \quad x\in [0, 1], \nonumber \\
&v_x(0,t,\varepsilon, c)=-i_0(t), &v_x(L,t,\varepsilon, c)=0, \quad t \in [0,5].
\end{align}
The external input $u(t)$ is $i_0(t)=5 \times 10^4 t^3 e^{(-15t)}$.
After discretization in space, we obtain a discretized system in the form of~(\ref{eq:fom}) with $N=16384$. The QoIs of this model are the two state variables on the left boundary, i.e., $\boldsymbol{y} := (v(0,t, \boldsymbol{\mu}), w(0,t, \boldsymbol{\mu}))^T$.

The parameters are sampled in a 2D parameter space $ [0.01, 0.04] \times [0.025, 0.075]$ via Latin hypercube sampling, leading to $n_p=126$ parameter samples, i.e., $\boldsymbol{\mu}_i=(\varepsilon_i, c_i)^T, i=1, \ldots, 126$. Given any sample $\boldsymbol{\mu^*} = (\varepsilon^*, c^*)^T$ of $\boldsymbol{\mu}$, we obtain the data from numerically solving the discretized system with fixed time step size $\Delta t=0.01$, resulting in a solution sequence with 500 time steps. The sequence of each output can be straightforwardly extracted from the solution sequence. The training, validation, and testing data are divided according to the parameter samples as $n_{p_{train}}=108$, $n_{p_{validate}}=12$ and $n_{p_{test}}=6$. To train SE-TFT, the outputs at the first time instance $\boldsymbol{y}(t_1, \boldsymbol{\mu})$ are treated as the past observed outputs. The outputs $\boldsymbol{y}(t_j, \boldsymbol{\mu})$, $j = 2, \ldots, 501$, in the following 500 time instances are to be predicted by SE-TFT. In the training phase, the time sequence corresponding to each parameter sample is not further divided into subsets. Training SE-TFT takes 8000 epochs without early stopping. Some other hyperparameters used in training SE-TFT are listed in \Cref{tab:hyper_FHN}. Again, SE-TFT predicts the output sequences at all future time instances in one step for any testing parameter. 
\begin{table}
\small
\caption{The FitzHugh-Nagumo model: the hyperparameters for training SE-TFT.}  \label{tab:hyper_FHN}
  \begin{tabular}{|c|c|c|c|c|c|} \hline
   Learning rate & Dropout rate & Number of heads & $d_{model}$ & Minibatch size & Max gradient norm \\ \hline
 0.0005  & 0.1 & 1 & 160 & 64 & 100\\ \hline
\end{tabular}
\end{table}

 For this numerical example, SE-TFT is trained to predict quantiles for $q \in \{0.1, 0.5, 0.9\}$ (corresponding to the P10, P50, and P90 forecasts), and the resulting dynamical behaviors together with the associated forecast intervals for six testing parameter cases are shown in \Cref{fig:FHN_sol_q}. The dash–dot curves represent the P50 forecasts, i.e., the output values for the two output variables, while the prediction interval is the green area bounded by the P10 and P90 forecasts in green dashed lines. As observed in the figure, the reference solutions are consistently inside the prediction intervals, indicating good uncertainty calibration over different dynamical patterns. 
 
 The output errors $\epsilon_k$ between the P50 forecasts and the reference solutions for 6 testing cases are reported in~\Cref{tab:error_FHN}. Averaged over all testing parameter samples, the mean relative errors are $\frac{1}{6}\sum_{k = 1}^{6}\epsilon_{v}(\boldsymbol{\mu}_k) = 0.0314$ and $\frac{1}{6}\sum_{k = 1}^{6}\epsilon_{w}(\boldsymbol{\mu}_k) = 0.0188$ for the two output variables, respectively. The P50 and P90 quantile losses ($\mathcal{L}_{q=0.5}, \, \mathcal{L}_{q=0.9}$) evaluated on the testing cases are 0.0585 and 0.0248, respectively. Overall, these results demonstrate that SE-TFT accurately captures both the parameter-dependent dynamics and predictive uncertainty using only the initial states and testing parameters as inputs.
 
\begin{figure}
\begin{subfigure}{0.33\textwidth}
\includegraphics[width=0.9\linewidth]{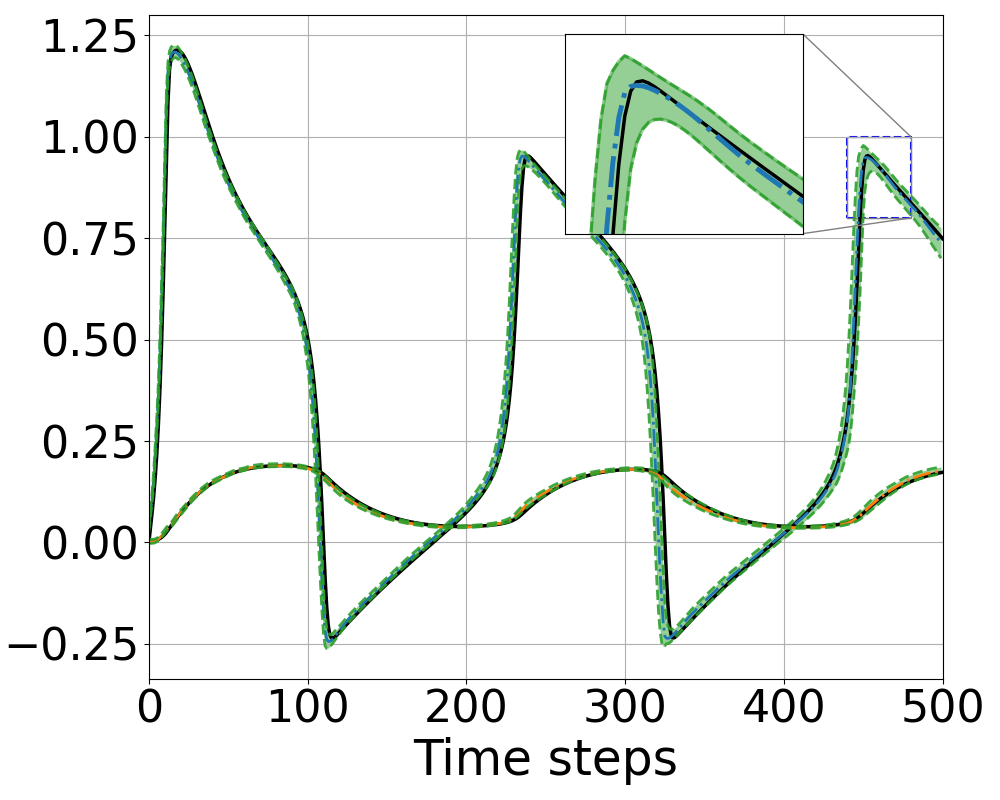}
\label{fig:FHN_1_q}
\end{subfigure}
\begin{subfigure}{0.33\textwidth}
\includegraphics[width=0.9\linewidth]{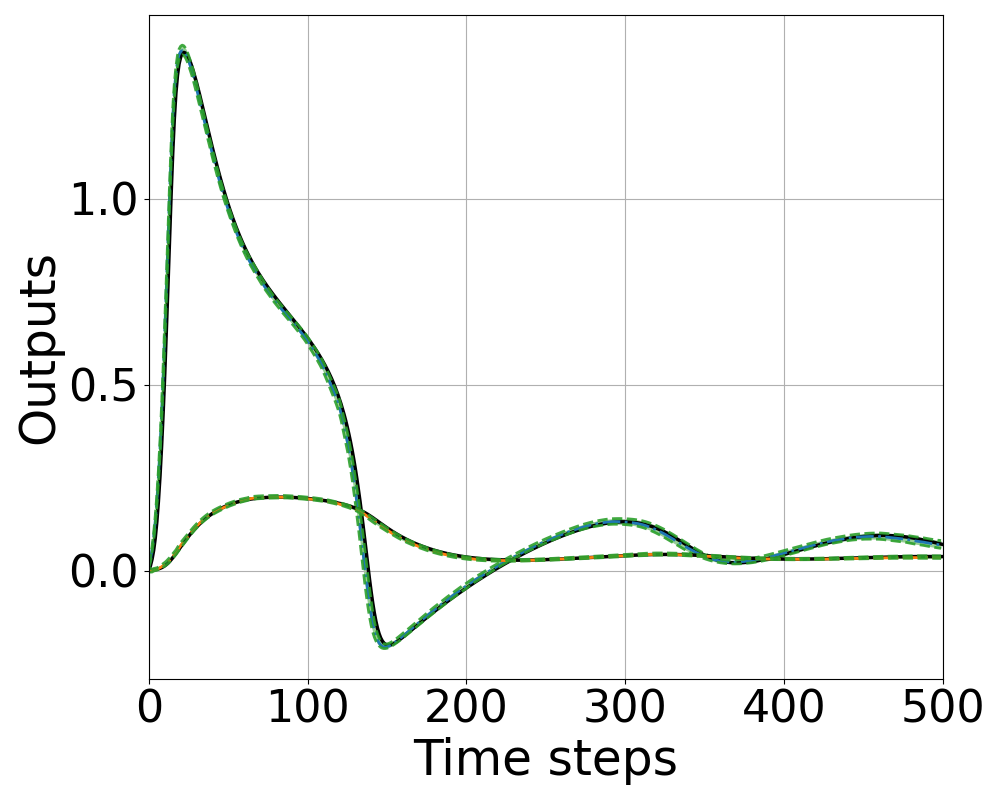}
\label{fig:FHN_2_q}
\end{subfigure}
\begin{subfigure}{0.33\textwidth}
\includegraphics[width=0.9\linewidth]{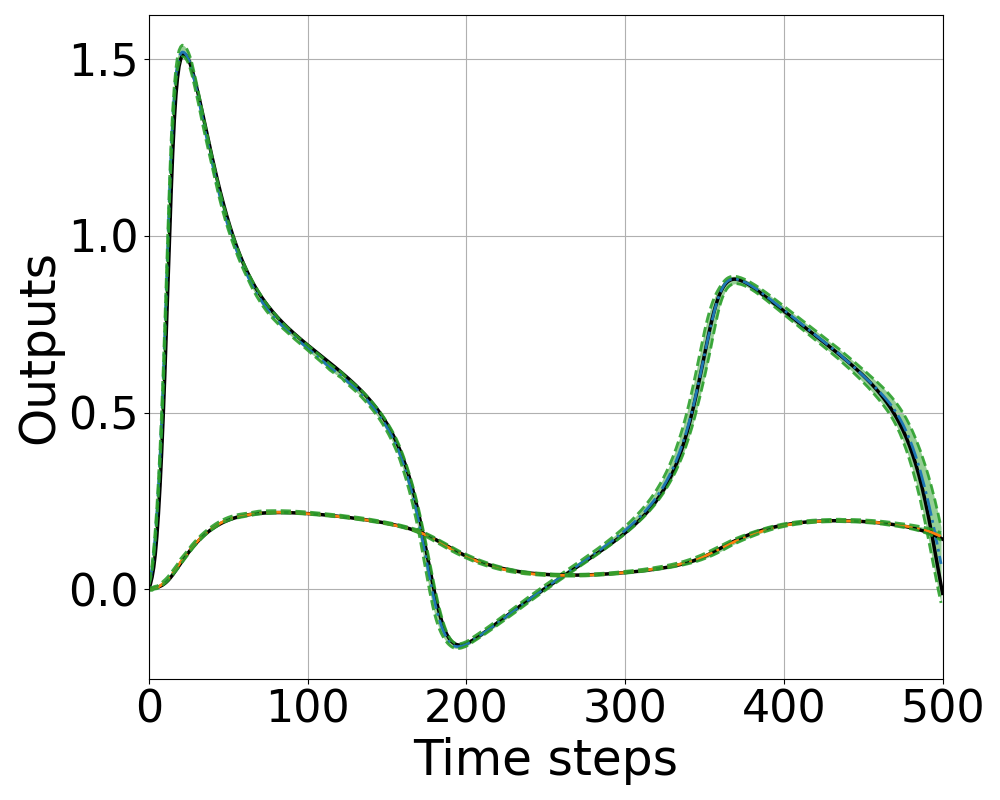}
\label{fig:FHN_3_q}
\end{subfigure}
\begin{subfigure}{0.33\textwidth}
\includegraphics[width=0.9\linewidth]{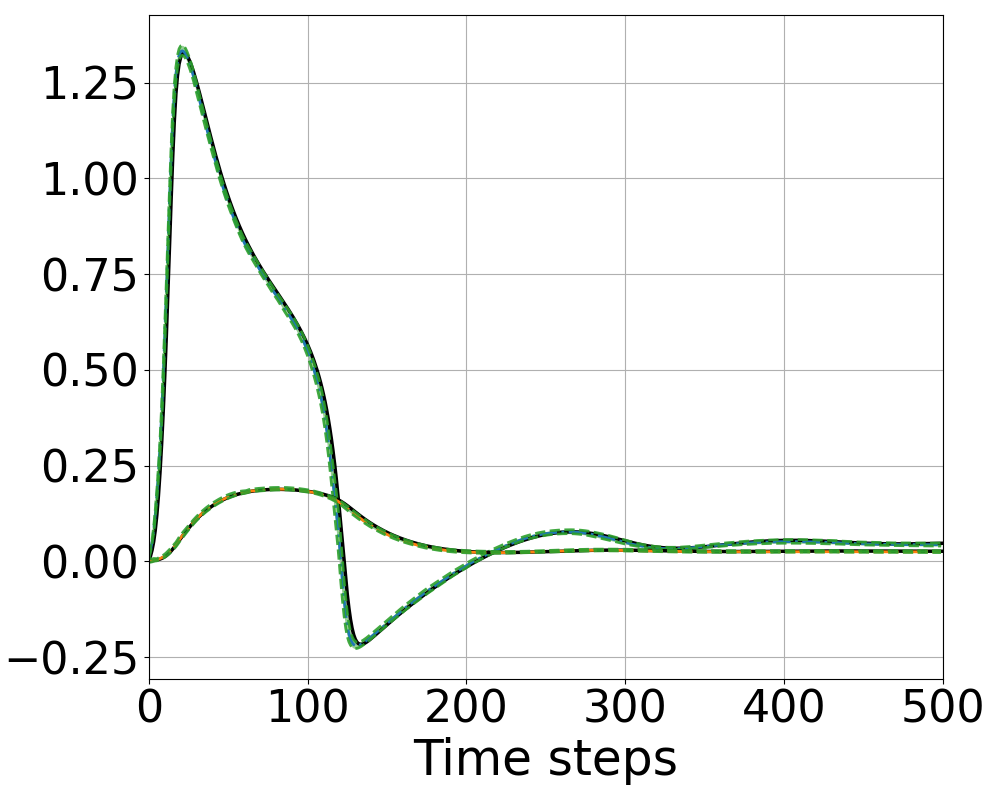}
\label{fig:FHN_4_q}
\end{subfigure}
\begin{subfigure}{0.33\textwidth}
\includegraphics[width=0.9\linewidth]{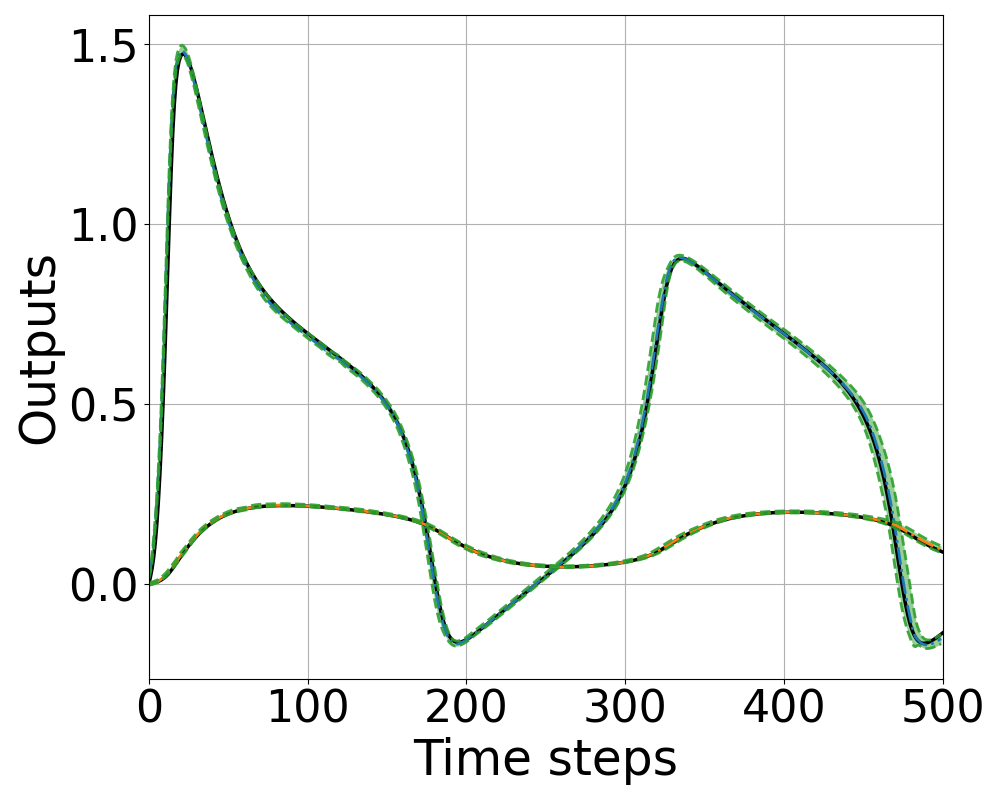}
\label{fig:FHN_5_q}
\end{subfigure}
\begin{subfigure}{0.33\textwidth}
\includegraphics[width=0.9\linewidth]{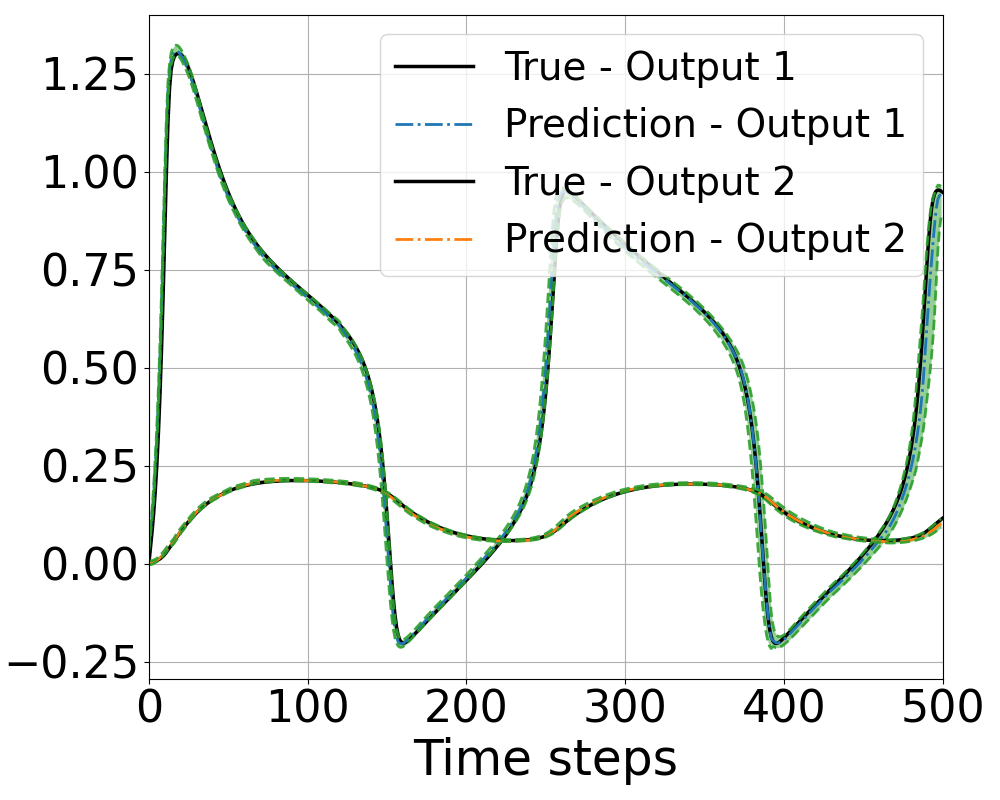}
\label{fig:FHN_6_q}
\end{subfigure}
\caption{FitzHugh-Nagumo model: the predicted solution (P50) with the quantile interval (P10, P90) and the reference solution at the testing parameter samples listed in~\Cref{tab:error_FHN}: the first three samples correspond to the figures (from left to right) on the top, respectively and the next three samples correspond to the figures (from left to right) in the bottom. }
\label{fig:FHN_sol_q}
\end{figure}
\begin{table}
\small
\centering
\caption{The FitzHugh–Nagumo model: mean relative error in~\labelcref{eq:err_ind} for 6 testing cases.}
\label{tab:error_FHN}
\begin{tabular}{@{}lcc@{}}
\toprule
Testing parameter samples & & \\
$\boldsymbol{\mu}^*_k = (\varepsilon^*_k, c^*_k)^T$, $k=1,\ldots,6$& $\epsilon_{v}(\boldsymbol{\mu}^*_k)$ & $\epsilon_{w}(\boldsymbol{\mu}^*_k)$ \\
\midrule
$(0.0125,\,0.0458)$ & 0.0594 & 0.0304 \\
$(0.0275,\,0.0375)$ & 0.0160 & 0.0173 \\
$(0.0375,\,0.0542)$ & 0.0221 & 0.0125 \\
$(0.0225,\,0.0292)$ & 0.0212 & 0.0205 \\
$(0.0325,\,0.0625)$ & 0.0188 & 0.0134 \\
$(0.0175,\,0.0708)$ & 0.0509 & 0.0189 \\
\midrule
\textit{Averaged over all 6 testing cases} & 0.0314 & 0.0188 \\
\bottomrule
\end{tabular}
\end{table}

\subsection{Coupled electrochemical kinetics and diffusion model}
This example describes the Ferrocyanide oxidation reaction in a coupled electrochemical kinetics and diffusion model from~\cite{morKocetal11}. The focus is on the reaction kinetics influenced by two parameters: the angular frequency of the input signal ($\omega$) and the rotation rate of the rotating disc electrode ($\omega_r$). The governing equations include the two PDEs in~\labelcref{eq:fc_fick}, which describe the mass transport of the oxidant and reductant according to the second Fick's law, assuming that convective terms can be neglected.
\begin{align}
\label{eq:fc_fick}
\frac{\partial C_\mathrm{ox}(z, t, \boldsymbol{\mu})}{\partial t}&=D_\mathrm{ox} \frac{\partial^{2} C_\mathrm{ox}(z, t, \boldsymbol{\mu})}{\partial z^{2}}, \nonumber \\
\frac{\partial C_\mathrm{red}(z, t, \boldsymbol{\mu})}{\partial t}&=D_\mathrm{red} \frac{\partial^{2} C_\mathrm{red}(z, t, \boldsymbol{\mu})}{\partial z^{2}},
\end{align}
where the subscript `ox' represents oxidant and `red' represents reductant in the reaction. The vector of parameters is $\boldsymbol{\mu} = (\omega, \omega_r)^T$. Both reduction and oxidation can occur in the system. The terms $C_{\ell}, D_{\ell}, \, \ell= \mathrm{ox}, \mathrm{red}$, correspond to the concentration and diffusion coefficient, respectively. Additionally, an ODE for the charge balance is provided in~\labelcref{eq:fc_charge},
\begin{equation}
\label{eq:fc_charge}
C_{dl} \frac{dE(t, \boldsymbol{\mu})}{dt} = J(E(t,\boldsymbol{\mu}), u(\omega,t)) - F_a \cdot r(t,\boldsymbol{\mu}),
\end{equation}
where $C_{dl}$ is the double-layer capacitance, $J(E(t, \boldsymbol{\mu}), u(\omega,t))$ is the cell current density depending on the electrode potential $E(t, \boldsymbol{\mu})$, and $F_a$ is the Faraday constant. The external input signal is the potential of the voltage depending on the angular frequency, i.e.,
\begin{equation}
u(\omega, t) = E_0 + A \cos\omega t,
\end{equation}
where $E_0=0.107$, the amplitude $A = 0.0536$, and the angular frequency $\omega \in [10\pi , 1000\pi]$\si{rad/s}. The boundary condition is given as
$$D_\ell \frac{\partial C_\ell (z, t, \boldsymbol{\mu})}{\partial z}\vert_{z=0}=\pm f_r(t, \boldsymbol{\mu}), \; \ell={\mathrm{red}, \mathrm{ox}},$$
where $f_r(t, \boldsymbol{\mu})$ is an exponential function of the system unknown variable $E(t)$, defined as
\begin{equation}
\label{eq:ferr_reactionrate}
f_r(t,\boldsymbol{\mu})= k_r\left\{c_\mathrm{red}(t, \boldsymbol{\mu}) e^{\beta g \cdot \left(E(t,\boldsymbol{\mu})-E_\mathrm{r}\right)} 
-c_\mathrm{ox}(t, \boldsymbol{\mu}) e^{-(1-\beta)g \cdot \left(E(t,\boldsymbol{\mu})-E_\mathrm{r}\right)}\right\}.
\end{equation}
Here, $c_\ell(t, \boldsymbol{\mu})=\frac{C_{\mathrm{\ell}}(0, t, \boldsymbol{\mu})}{C_{\mathrm{\ell}, \infty}(t, \omega_r)}, \ell={\mathrm{red}, \mathrm{ox}}$, 
$g=F/RT$ with $T$ being the temperature, and R being the universal gas constant. The variable $C_{\mathrm{\ell}, \infty}(t, \omega_r)$ is the bulk concentration changing with time and the rotation rate $\omega_r$. The reaction rate $f_r(t, \boldsymbol{\mu})$ in~\labelcref{eq:ferr_reactionrate} is computed by Butler-Volmer kinetics, where $E_r$ is the equilibrium electrode potential and $k_r=1.15\times 10^{-4}$ is the reaction rate constant.

After discretization in space, the total number of DOFs is $N=2003$. The resulting discretized model forms a system of ODEs as described in~(\ref{eq:fom}), with $E=I$, the identity matrix. The nonlinear vector $F(x(t,\boldsymbol{\mu}), \boldsymbol{\mu})$ is an exponential function of the state $E(t, \boldsymbol{\mu})$. The boundary conditions also contribute to $F(x(t, \boldsymbol{\mu}), \boldsymbol{\mu})$. There are three outputs: the current density $J(E(t, \boldsymbol{\mu}), u(\omega,t))$, the concentration of the oxidant $C_\mathrm{ox}(0, t, \boldsymbol{\mu})$ and the concentration of the reductant $C_\mathrm{red}(0, t, \boldsymbol{\mu})$ on the boundary. Finally, the output vector is $\boldsymbol{y}(\boldsymbol{\mu}, t)=(J(E, u(\omega,t)), C_\mathrm{ox}(0, t, \boldsymbol{\mu}), C_\mathrm{red}((0, t, \boldsymbol{\mu}))^T$. Using the relation between the ordinary frequency $f$ and the angular frequency $\omega$, $\omega=2\pi f$, we sample $f$ to determine the samples of $\omega$.

To train SE-TFT, we take samples in a 2D parameter space $[5, 500] \times [500, 5000]$ using Latin hypercube sampling, resulting in $n_p = 450$ parameter samples with $\boldsymbol{\mu}_i = (f_i, (\omega_r)_i)^T, i = 1, \ldots, 450$, where the training, validation, and testing data include $n_{p_{train}} = 400$, $n_{p_{validate}} = 40$ and $n_{p_{test}} = 10$ parameter samples, respectively. The simulation time interval for each parameter sample is defined as 5 periods with each period containing 100 evenly distributed time steps, resulting in a total of 500 time instances in the time interval $[0, 5/f_i]$ changing according to the sample value $f_i, i=1,\ldots, 450$. This means that although each parameter sample corresponds to a time sequence with the same number (500) of time steps, the time interval from which the time sequence is obtained is different and is determined by the frequency sample value. The higher the frequency, the shorter the time interval, resulting in time sequences changing at both high and low frequencies. Using the samples of $\boldsymbol{\mu}$, the output at the first time instance $\boldsymbol{y}(\boldsymbol{\mu}, t_1)=(J(E, u(\omega,t_1)), C_\mathrm{ox}(0, t_1, \boldsymbol{\mu}), C_\mathrm{red}(0, t_1, \boldsymbol{\mu}))^T$ as the past observed output, and the known input signal $u(\omega, t)$ at all the time instances, we construct the data file in~\labelcref{eq:motft_data} to train SE-TFT. 


During the training phase, SE-TFT is trained with MAE loss to predict output values at the subsequent 499 time instances. The training process involves 5000 epochs with no early stopping. Other hyperparameters are listed in \Cref{tab:hyper_Ferr}. During the testing phase, given only the initial condition for any testing parameter sample, SE-TFT can predict the multiple outputs at the next 499 time instances $\{t_2,\ldots, t_{n_t}\}$ in one step. 
\begin{table}
\small
\centering
\caption{Coupled electrochemical kinetics and diffusion model: the hyperparameters for training SE-TFT.}  \label{tab:hyper_Ferr}
  \begin{tabular}{|c|c|c|c|c|c|} \hline
   Learning rate & Dropout rate & Number of heads & $d_{model}$ & Minibatch size & Max gradient norm \\ \hline
 0.0005  & 0.2 & 4 & 160 & 64 & 0.01\\ \hline
\end{tabular}
\end{table}

The prediction results for three representative testing cases are shown in~\Cref{fig:ferr_sol_q}. The mean relative errors $\epsilon_{y}(\boldsymbol{\mu})$ between the predictions and the reference solutions for all 10 testing cases are listed in~\Cref{tab:error_Ferr}. Averaged over the 10 testing cases, the values of the mean relative error are $\frac{1}{10}\sum_{k = 1}^{10}\epsilon_{J}(\boldsymbol{\mu}_k) = 0.0017$, $\frac{1}{10}\sum_{k = 1}^{10}\epsilon_{C_\mathrm{ox}}(\boldsymbol{\mu}_k) = 0.0068$ and $\frac{1}{10}\sum_{k = 1}^{10}\epsilon_{C_\mathrm{red}}(\boldsymbol{\mu}_k) = 0.0024$ for the three output predictions, respectively. 
\begin{figure}
\begin{subfigure}{0.33\textwidth}
\includegraphics[width=0.9\linewidth]{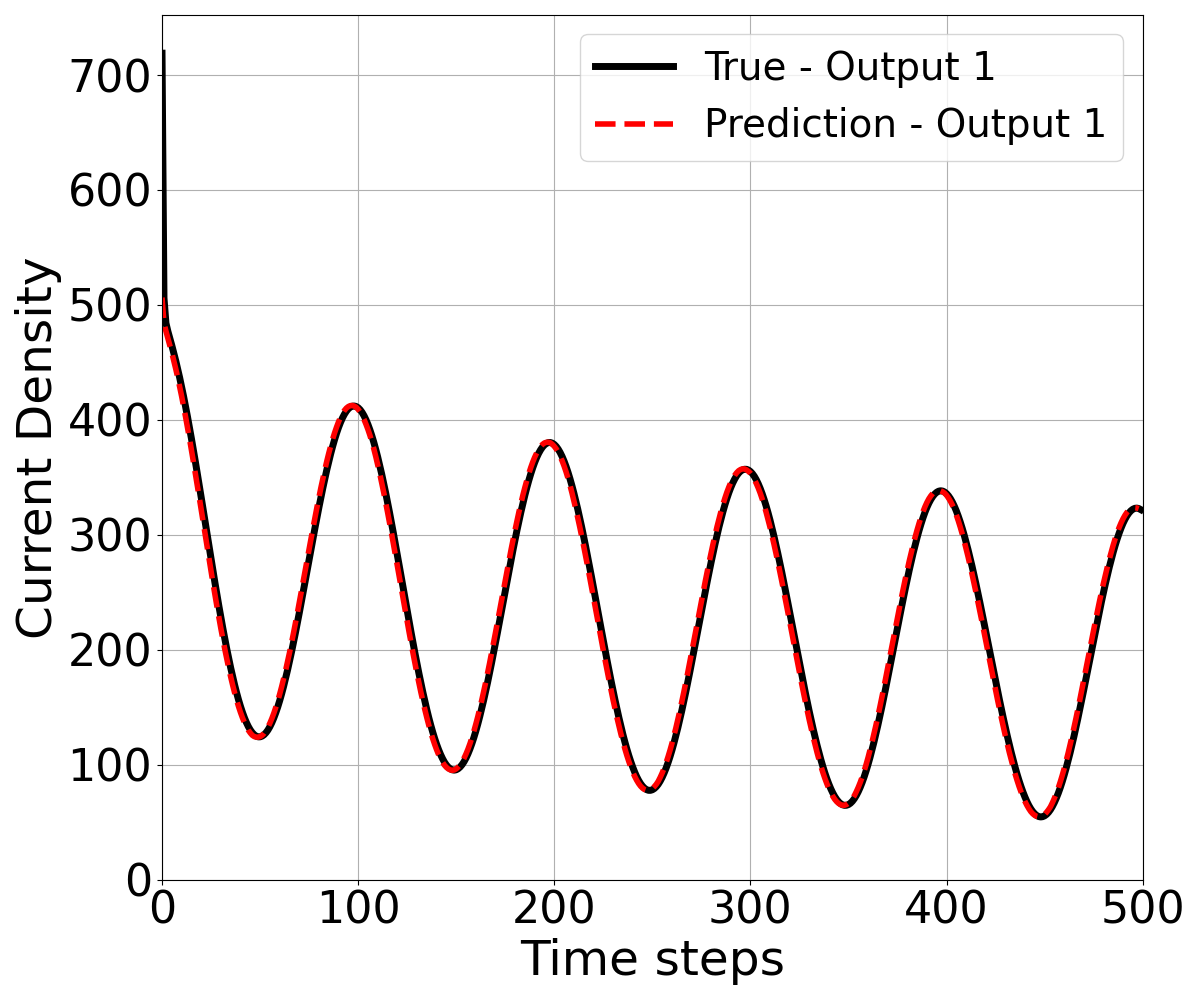}
\label{fig:Ferr_1_i_q}
\end{subfigure}
\begin{subfigure}{0.33\textwidth}
\includegraphics[width=0.9\linewidth]{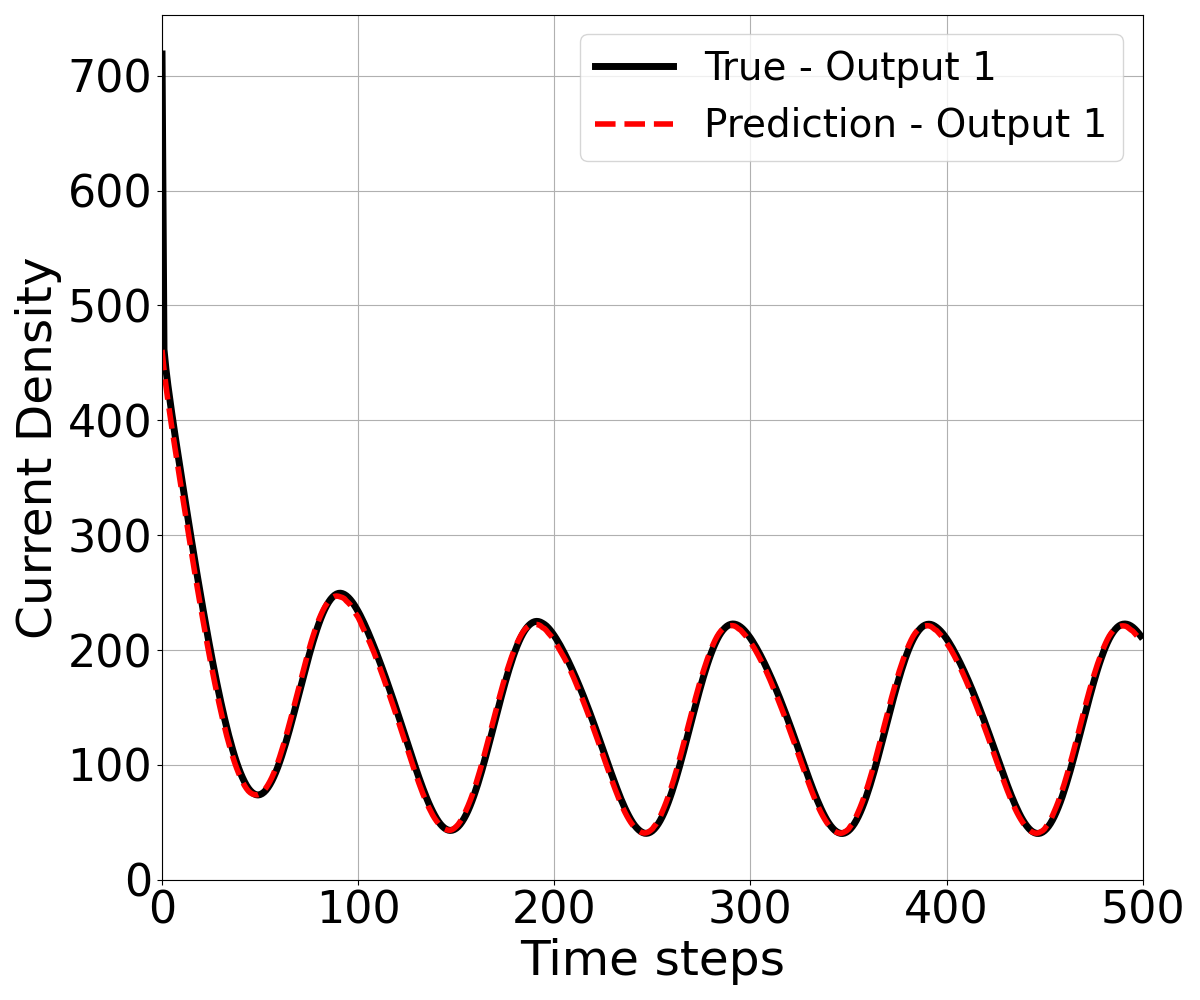}
\label{fig:Ferr_2_i_q}
\end{subfigure}
\begin{subfigure}{0.33\textwidth}
\includegraphics[width=0.9\linewidth]{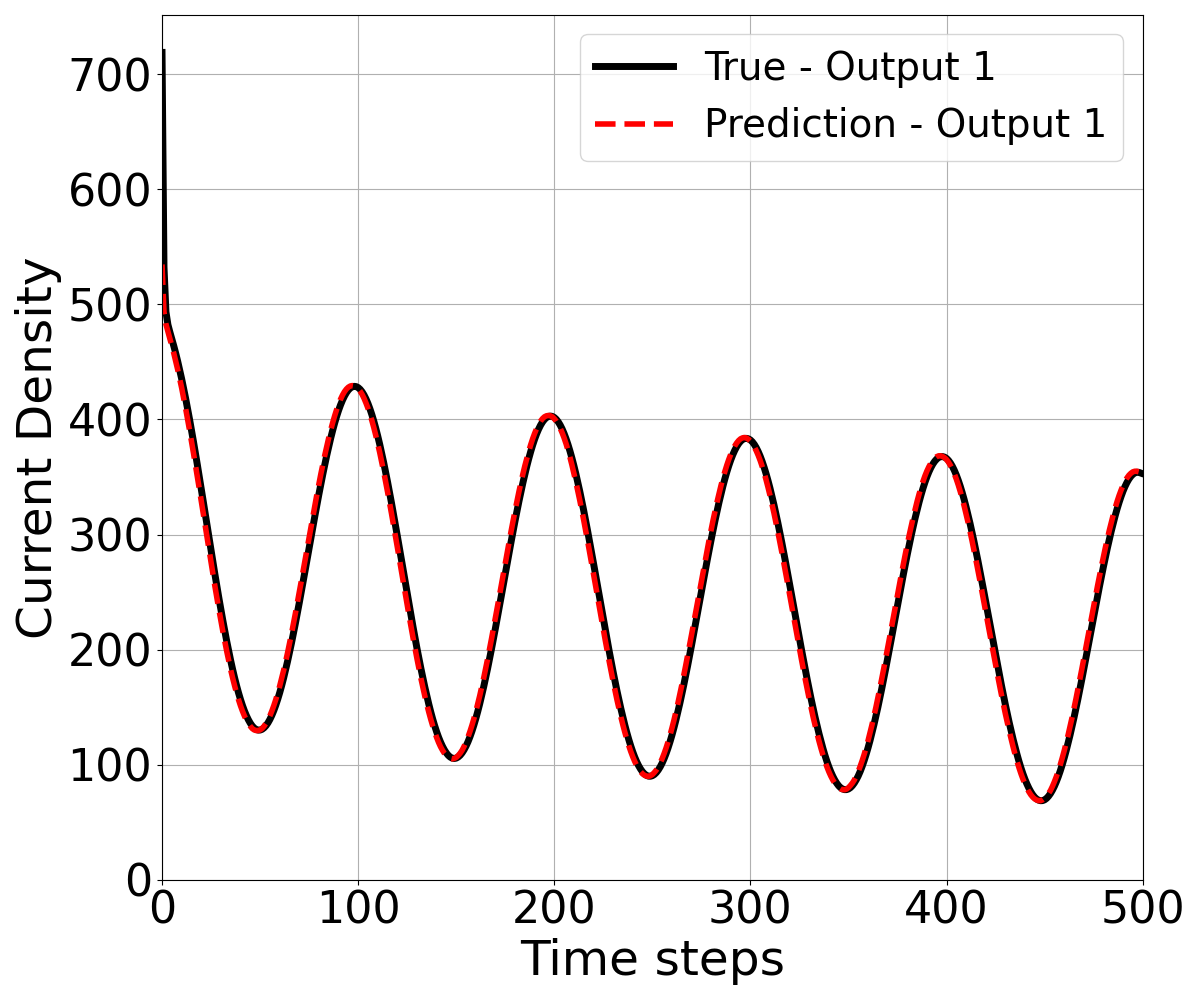}
\label{fig:Ferr_7_i_q}
\end{subfigure}
\begin{subfigure}{0.33\textwidth}
\includegraphics[width=0.9\linewidth]{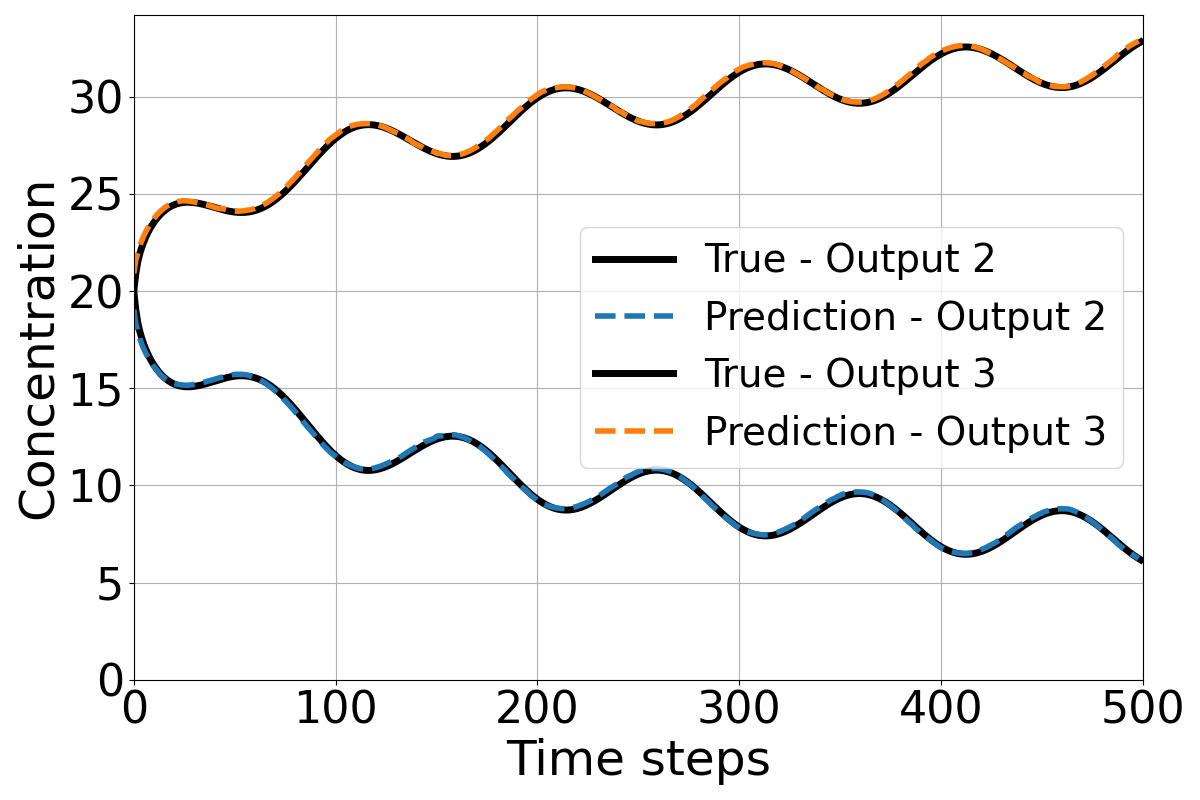}
\label{fig:Ferr_1_con_q}
\end{subfigure}
\begin{subfigure}{0.33\textwidth}
\includegraphics[width=0.9\linewidth]{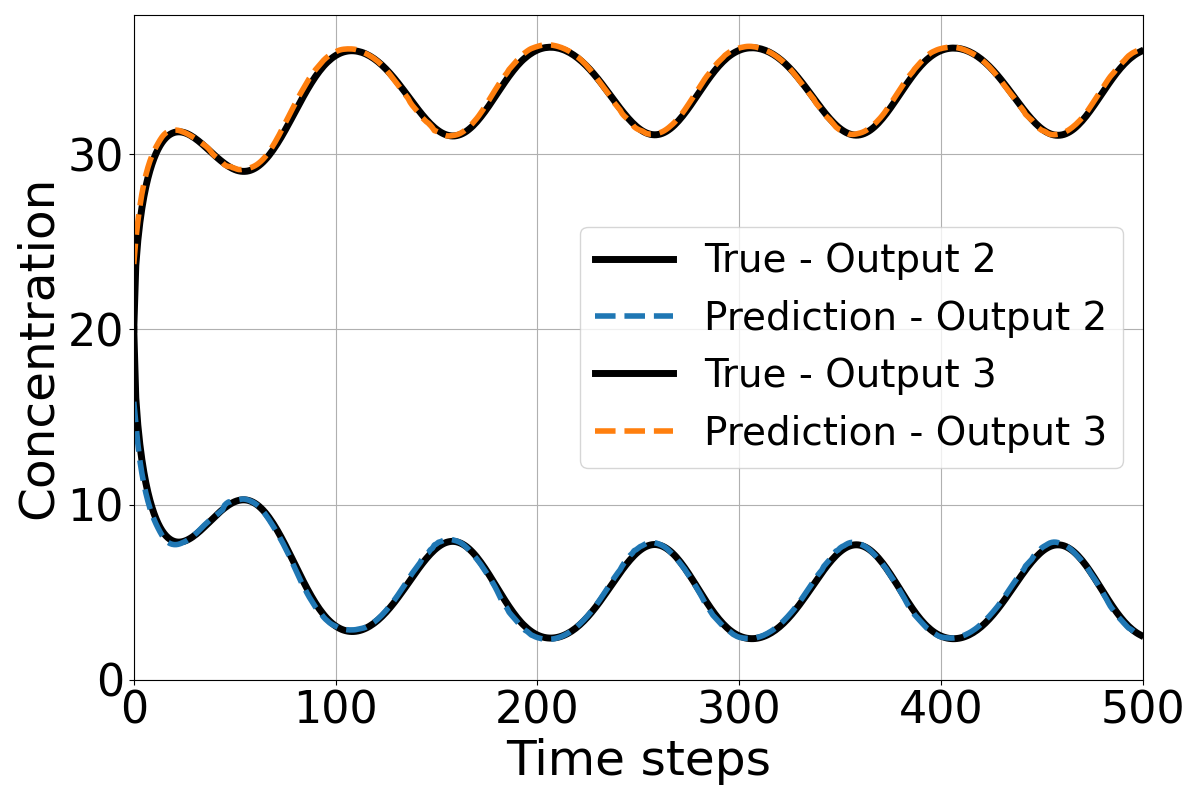}
\label{fig:Ferr_2_con_q}
\end{subfigure}
\begin{subfigure}{0.33\textwidth}
\includegraphics[width=0.9\linewidth]{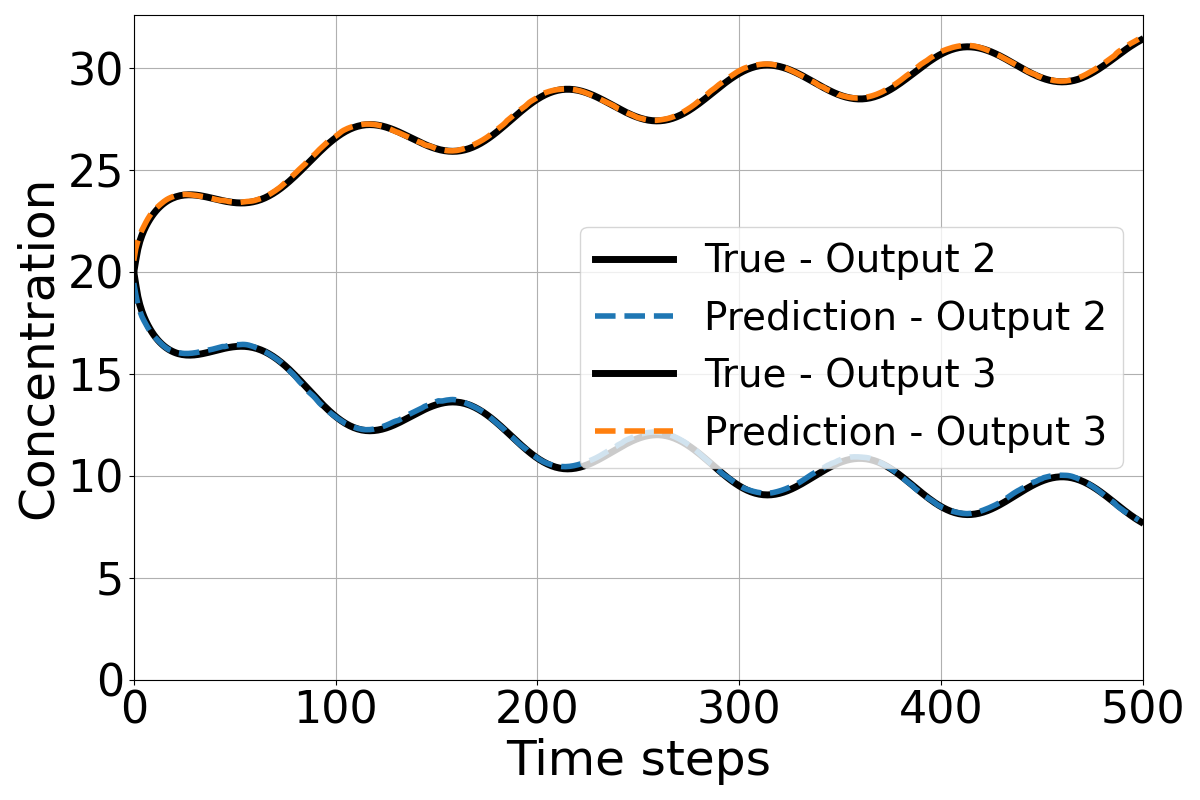}
\label{fig:Ferr_7_con_q}
\end{subfigure}
\caption{Coupled electrochemical kinetics and diffusion model: the predicted solution and the reference solution. The upper part shows the predicted output of the current density. The lower part is the predicted output of concentration for the reduced and oxidized form. These results correspond to three testing parameter samples $\boldsymbol{\mu}_k, \, k = 1,2,7$, in~\Cref{tab:error_Ferr}, respectively.}
\label{fig:ferr_sol_q}
\end{figure}

\begin{table}
\small
\centering
\caption{Coupled electrochemical kinetics and diffusion model: mean relative error $\epsilon_{y}(\boldsymbol{\mu}^*_k)$ in~\labelcref{eq:err_ind} for 10 testing cases.}
\label{tab:error_Ferr}
\begin{tabular}{@{}lccc@{}}
\toprule
Testing parameter samples \\ 
$\boldsymbol{\mu}^*_k = (f_k, (\omega_r)_k)^T$, $k=1,\ldots,10$ 
& $\epsilon_{J}(\boldsymbol{\mu}^*_k)$ 
& $\epsilon_{C_{\mathrm{ox}}}(\boldsymbol{\mu}^*_k)$ 
& $\epsilon_{C_{\mathrm{red}}}(\boldsymbol{\mu}^*_k)$ \\
\midrule
(295.5641, 4649.3851) & 0.0013 & 0.0060 & 0.0025 \\
(29.9218, 3922.5145)  & 0.0053 & 0.0124 & 0.0032 \\
(91.1469, 1902.9533)  & 0.0018 & 0.0070 & 0.0024 \\
(223.2115, 758.4492)  & 0.0010 & 0.0048 & 0.0026 \\
(488.7329, 1190.1299) & 0.0018 & 0.0078 & 0.0011 \\
(304.2217, 1658.7809) & 0.0012 & 0.0056 & 0.0025 \\
(447.3810, 3445.1554) & 0.0016 & 0.0069 & 0.0016 \\
(193.9293, 2437.0589) & 0.0009 & 0.0045 & 0.0028 \\
(122.6000, 3156.9897) & 0.0012 & 0.0074 & 0.0025 \\
(386.5754, 4219.9428) & 0.0010 & 0.0054 & 0.0026 \\
\midrule
\textit{Averaged over all 10 testing cases} & 0.0017 & 0.0068 & 0.0024 \\
\bottomrule
\end{tabular}
\end{table}

\subsection{The interpretability analysis of SE-TFT}
\label{sec:Interpretable}
In this section, we present the interpretability of SE-TFT from two perspectives: variable importance obtained from the variable selection layers, and the temporal-spatial correlation obtained from the block-wise masked attention weight matrices.

Variable importance weights measure the influences of parameters, past inputs and past observed outputs, as well as the future inputs, on the final prediction of the time sequence. \Cref{tab:varselect} shows the variable importance weights for the second and the third example, respectively, where the sum of the weights of all parameters is $1.0$, while the sum of the past input weight and the observed output weight in the past time period corresponding to each output is $1.0$. Finally, we only have a single input signal for both numerical examples, so the weights of the future input equal to $1.0$. More specifically, the weights of the parameters indicate their influences on prediction of all the outputs $\{y_1, \ldots, y_{n_o}\}$, while the weights of the inputs and outputs in the past period indicate their influences on each individual output prediction. For the FitzHugh-Nagumo model, the parameter $\varepsilon$ shows a slightly greater influence on the final predicted output sequence than the parameter $c$ does. In the past time period $[0, t_1]$, the observed outputs contribute more than the known input signal $i_0(t)$ does for each individual output prediction, which remain unchanged from one output prediction to another. As for the coupled electrochemical kinetics and diffusion model, the frequency $f$ appears as a more significant parameter than the rotation rate $\omega_r$ after the static covariate selection process. The prediction of the output $(y_1)$ corresponding to the current density is more sensitive to the known inputs than to the observed ones. The reverse phenomenon happens to $y_2$ and $y_3$. Moreover, for both $y_2$ and $y_3$, the relative contributions of the external input $u(\omega, t)$ and the observed outputs remain unchanged. Such insights cannot be detected by the original TFT model, highlighting an advantage of the proposed SE-TFT framework. For the Lorenz–63 model, no parameters and no external inputs are involved in the numerical tests. Accordingly, we only consider the dynamics changing with random initial conditions. When training SE-TFT for this example, we take the initial conditions as the only observed outputs in the data set. As a result, the importance weights are reduced to a single importance weight related to the observed outputs, which is always $1.0$.

\begin{table}
\small
\caption{Variable importance for the first testing case of the FitzHugh-Nagumo model and for the first testing case of the coupled electrochemical kinetics and diffusion model.}
\label{tab:varselect}
\centerline{
\begin{subtable}{0.45\textwidth}
\hspace{-0.5cm}
\begin{tabular}{@{}lcc@{}}
\toprule
\textbf{}  & \multicolumn{2}{c}{\textbf{Importance weights}} \\ \midrule
{\ul \textbf{Parameters}} & \multicolumn{2}{c}{} \\ 
$\varepsilon$  & \multicolumn{2}{c}{0.5492} \\
$c$            & \multicolumn{2}{c}{0.4508}\\ \midrule
{\ul \textbf{Past}}  & \multicolumn{2}{c}{}\\
 & $y_1$ & $y_2$\\
Known inputs $i_0(t)$   & 0.3975 & 0.3975\\
Observed outputs   & 0.6205 & 0.6205\\ \midrule
{\ul \textbf{Future}} & \multicolumn{2}{c}{} \\ 
Known inputs $i_0(t)$  & \multicolumn{2}{c}{1.0} \\ \bottomrule
\end{tabular}
\caption{FitzHugh-Nagumo model}
\end{subtable} \hfill
\begin{subtable}{0.45\textwidth}
\hspace{-0.5cm}
\begin{tabular}{@{}lccc@{}}
\toprule
\textbf{}  & \multicolumn{3}{c}{\textbf{Importance weights}} \\ \midrule
{\ul \textbf{Parameters}} & \multicolumn{3}{c}{}\\
$f$  & \multicolumn{3}{c}{0.8365}\\
$\omega_{r}$  & \multicolumn{3}{c}{0.1635}\\ \midrule
{\ul \textbf{Past}}  & \multicolumn{3}{c}{} \\
& $y_1$ & $y_2$ & $y_3$\\
Known inputs $u(\omega, t)$   & 0.9133 & 0.1179& 0.1179\\
Observed outputs   & 0.0867 &0.8821 & 0.8821\\ \midrule
{\ul \textbf{Future}} & \multicolumn{3}{c}{} \\ 
Known inputs $u(\omega, t)$  & \multicolumn{3}{c}{1.0} \\ \bottomrule
\end{tabular}
\caption{Coupled electrochemical kinetics and diffusion model}
\end{subtable}}
\end{table}
The attention weight matrices $\widetilde{\boldsymbol{A}}$ from the trained SE-TFT reveal informative hidden connections inside output sequences of the dynamical system. Attention weight matrices corresponding to two testing cases for the Lorenz-63 model are shown in~\Cref{fig:mask_lorenz}. For each testing case, we display the attention weight matrix of the feature sequence up to the $27$-th time instance, in order to save space. The zoomed-in $9 \times 9$ region in the left figure of~\Cref{fig:mask_lorenz_zoom} shows the attention correlations among the features within three time instances. Each time instance corresponds to a $3 \times 3$ block, where direct interactions between features corresponding to different outputs are clearly observed. In particular, darker colors in the $(i,j)$-th entry indicate stronger interactions between the features of the $i$-th and $j$-th outputs. Such information might be neglected by existing transformer models that rely on concatenated multi-head attention. Moreover, these interactions among multiple outputs cannot be identified or explored by the original TFT model either. Additionally, \Cref{fig:mask_lorenz_t} illustrates temporal patterns of attention weights for the three predicted outputs at the 450th time step, revealing the spatio-temporal dependencies within the forecasting horizon. For instance, the left plot shows how the previous 128 time steps of all the three outputs contribute to predicting $y_1(t_{n_t},\boldsymbol{\mu}_1)$. In this plot, the three curves correspond to the contributions from three different outputs at the past 128 time steps, $\{y_1(t_{1},\boldsymbol{\mu}_1),\ldots,y_1(t_{n_t-1},\boldsymbol{\mu}_1)\}$, $\{y_2(t_{1},\boldsymbol{\mu}_1),\ldots,y_2(t_{n_t-1},\boldsymbol{\mu}_1)\}$, and $\{y_3(t_{1},\boldsymbol{\mu}_1),\ldots,y_3(t_{n_t-1},\boldsymbol{\mu}_1)\}$. Similarly, the middle and right plots visualize the corresponding attention weights for predicting $y_2(t_{n_t},\boldsymbol{\mu}_1)$ and $y_3(t_{n_t},\boldsymbol{\mu}_1)$, respectively. Markedly different temporal patterns are observed in the additional example shown in \Cref{fig:mask_lorenz_t1}, which relates to another testing case with a totally different initial condition. Overall, no strong or persistent temporal dependencies can be identified, as attention weights are dispersed irregularly across prior time steps. These observations reflect the underlying complexity of cross interactions in chaotic dynamics, demonstrating the capacity of SE-TFT to capture such behavior.

\begin{figure}
\begin{subfigure}{0.99\textwidth}
\centering
\includegraphics[scale = 0.63]{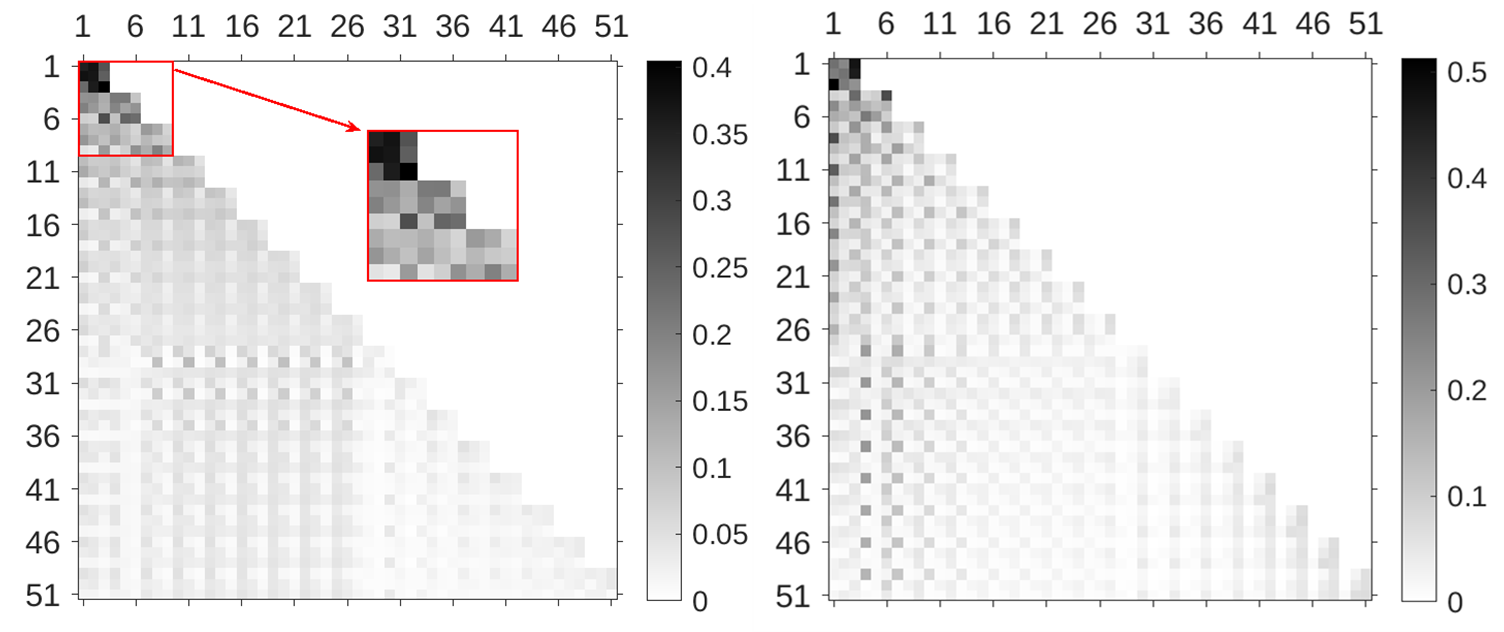}
\caption{}
\label{fig:mask_lorenz_zoom}
\end{subfigure}
\begin{subfigure}{0.99\linewidth}
  \centering
\begin{subfigure}{0.31\textwidth}
\includegraphics[width=0.93\linewidth]{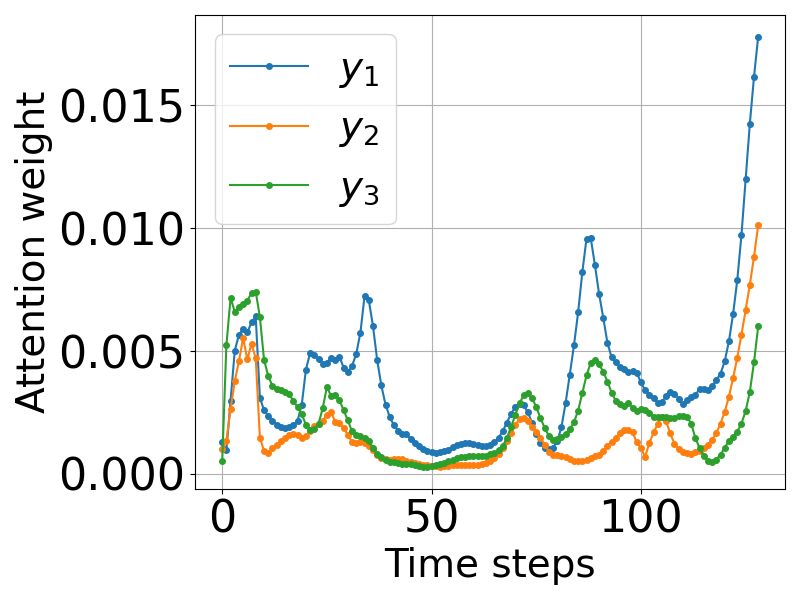}
\end{subfigure}
\begin{subfigure}{0.31\textwidth}
\includegraphics[width=0.93\linewidth]{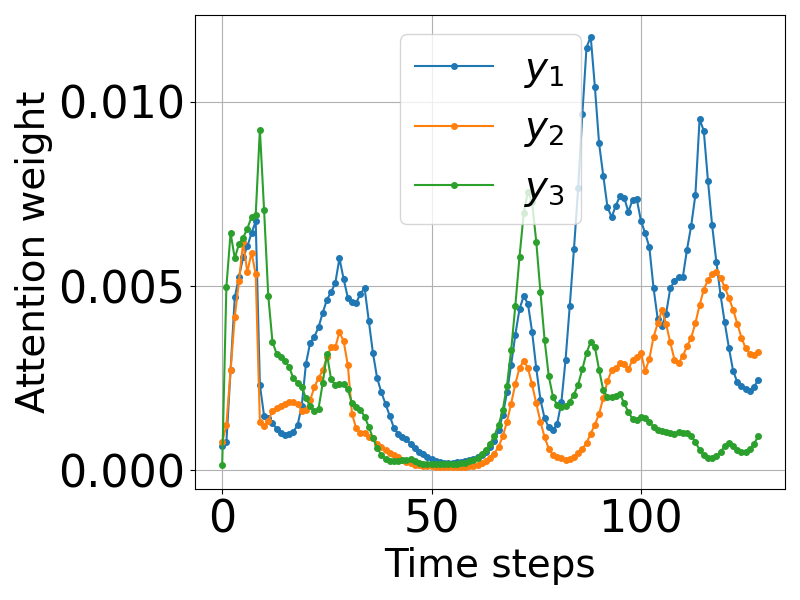}
\end{subfigure}
\begin{subfigure}{0.31\textwidth}
\includegraphics[width=0.93\linewidth]{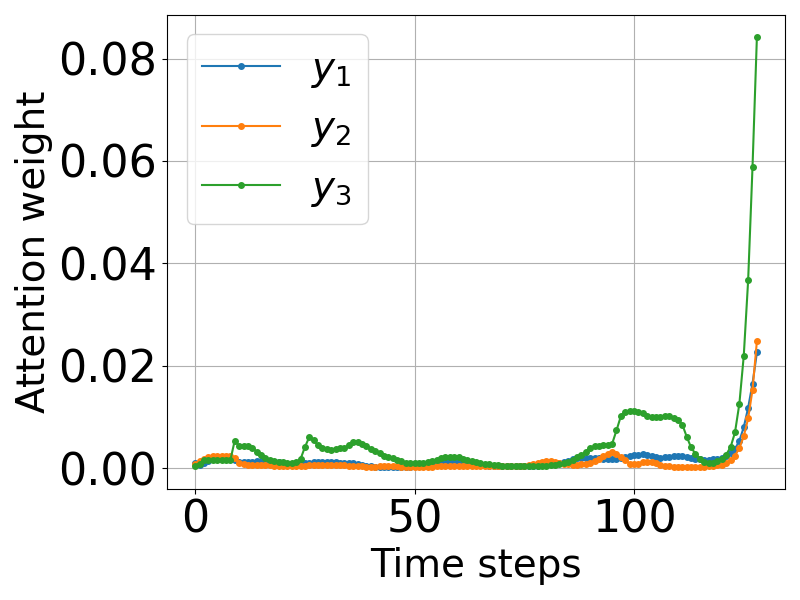}
\end{subfigure}
\caption{}
\label{fig:mask_lorenz_t}
\end{subfigure}
\begin{subfigure}{0.99\linewidth}
  \centering
\begin{subfigure}{0.31\textwidth}
\includegraphics[width=0.93\linewidth]{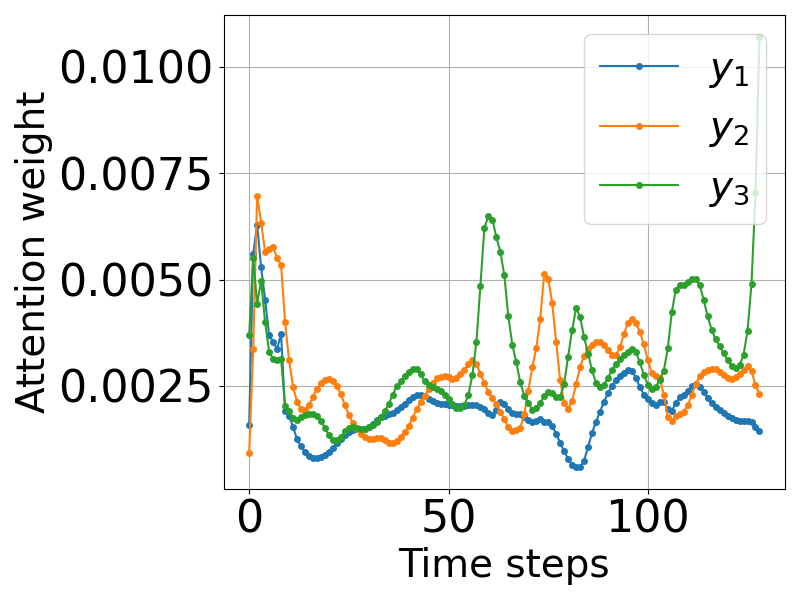}
\end{subfigure}
\begin{subfigure}{0.31\textwidth}
\includegraphics[width=0.93\linewidth]{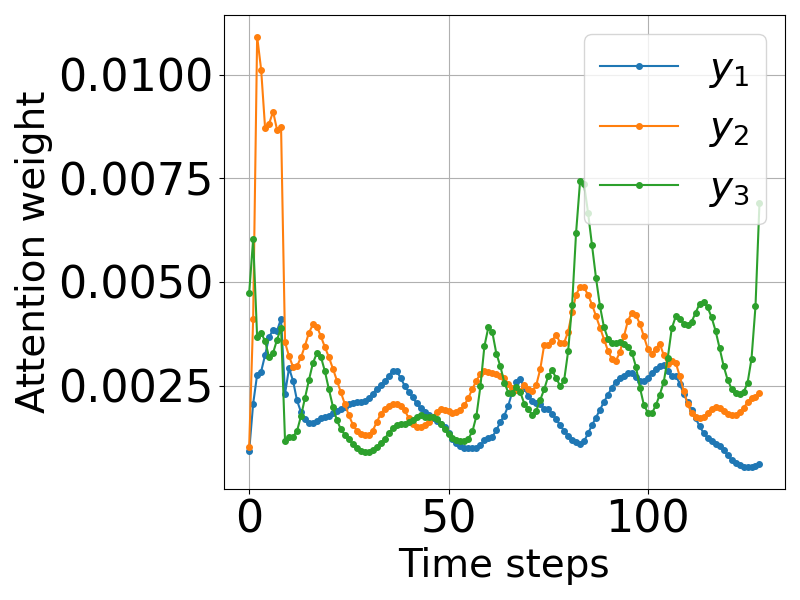}
\end{subfigure}
\begin{subfigure}{0.31\textwidth}
\includegraphics[width=0.93\linewidth]{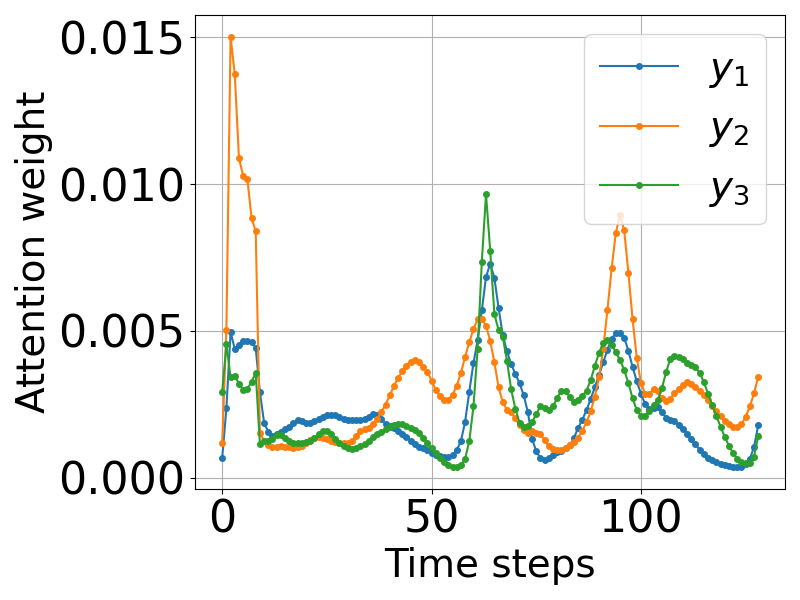}
\end{subfigure}
\caption{}
\label{fig:mask_lorenz_t1}
\end{subfigure}
\caption{The Lorenz-63 model: (a) The upper block (containing time steps up to the $17$-th time instance) of the block-wise masked attention weight matrices $\widetilde{\boldsymbol{A}}$ for predicting the outputs $\boldsymbol{y}(t, \boldsymbol{\mu})$ at (upper left)~the $450$-th and (upper right)~the $257$-th testing cases. Both results are obtained from the SE-TFT trained with MAE loss. (b) Temporal pattern of attention weights for predicting each of the three outputs at the final time step of the prediction horizon (the $450$-th testing case). (c) Temporal pattern of attention weights for predicting each of the three outputs at the final time step of the prediction horizon (the $257$-th testing case).}
\label{fig:mask_lorenz}
\end{figure}

\Cref{fig:mask_Ferr} illustrates the upper block of the attention weight matrices for the coupled electrochemical kinetics and diffusion model. In this numerical test, SE-TFT identifies clear spatio-temporal patterns. The weight matrix contains the correlation information up to the $33$-rd time instance. Based on this attention weight matrix, the feature associated with the current density is partially decoupled from the other two features corresponding to the concentrations. In the rows associated with current density, the strongest correlations (dark entries) occur in those entries linked to the current density itself and its own history. This behavior is further illustrated in \Cref{fig:mask_Ferr_t}. In the left picture, the prediction of $y_1(t_{n_t},\boldsymbol{\mu}_1)$ is dominated almost exclusively by past values of $y_1$ over the previous 499 time steps, as indicated by periodic spikes with increasing amplitude. This suggests that $y_1$ is the primary factor in forecasting $y_1$. In contrast, the middle and right figures show that $y_2(t_{n_t},\boldsymbol{\mu}_1)$ and $y_3(t_{n_t},\boldsymbol{\mu}_1)$ draw attention from all three outputs, indicating stronger cross-variable coupling. The lower plots in \Cref{fig:mask_Ferr_t1}, corresponding to a different testing parameter sample, exhibit a similar dynamic pattern with periodic structures and the distinguishable dominant influences on each output. These patterns differ from those observed in \Cref{fig:mask_lorenz_t} and \Cref{fig:mask_lorenz_t1} and can be captured only by SE-TFT.
\begin{figure}
\centering
\begin{subfigure}{0.99\textwidth}
\centering
\includegraphics[scale=0.3]{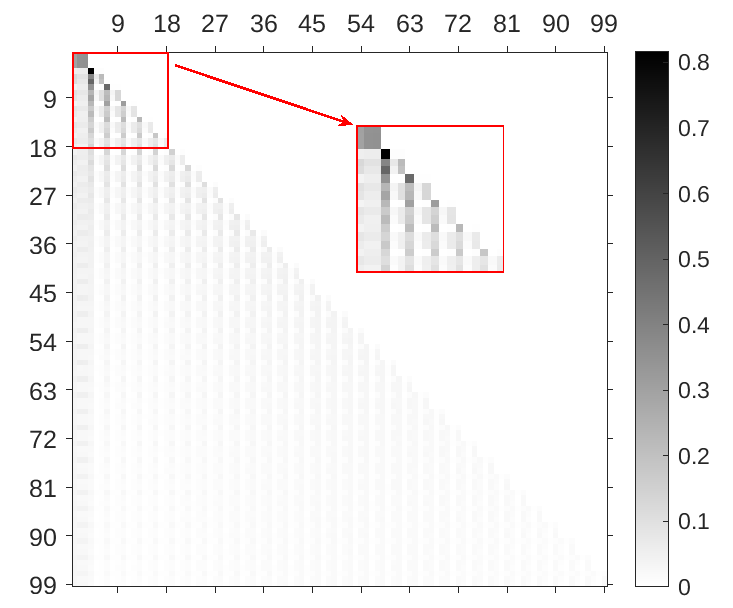}\caption{}
\label{fig:mask_Ferr}
\end{subfigure}
\begin{subfigure}{0.99\linewidth}
  \centering
\begin{subfigure}{0.31\textwidth}
\includegraphics[width=0.93\linewidth]{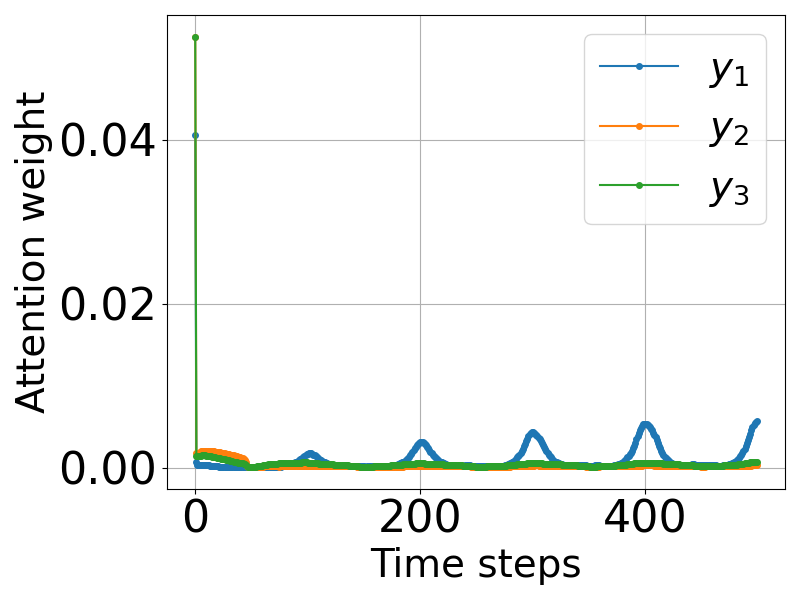}
\end{subfigure}
\begin{subfigure}{0.31\textwidth}
\includegraphics[width=0.93\linewidth]{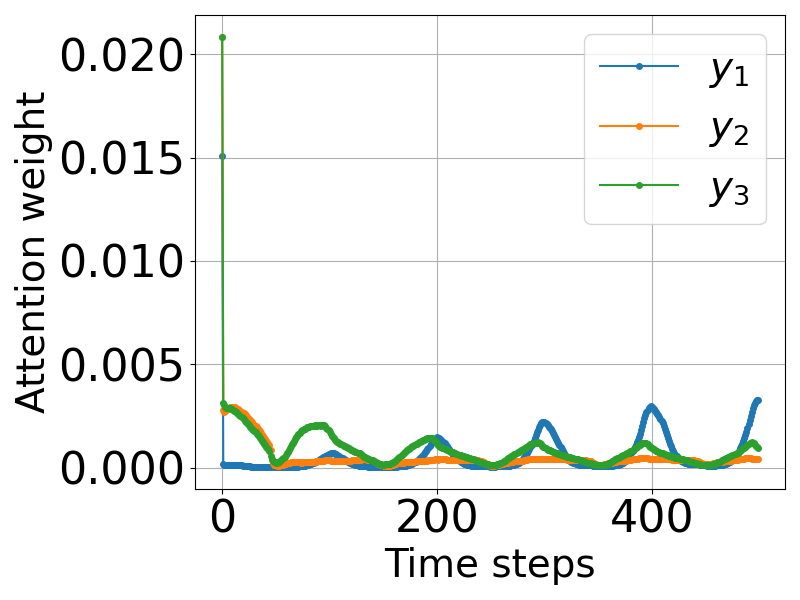}
\end{subfigure}
\begin{subfigure}{0.31\textwidth}
\includegraphics[width=0.93\linewidth]{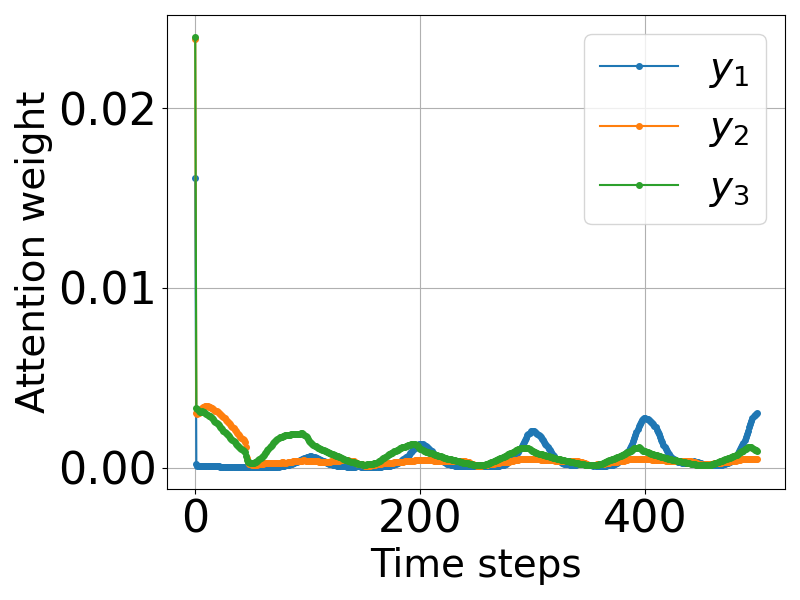}
\end{subfigure}
\caption{}
\label{fig:mask_Ferr_t}
\end{subfigure}
\begin{subfigure}{0.99\linewidth}
  \centering
\begin{subfigure}{0.31\textwidth}
\includegraphics[width=0.93\linewidth]{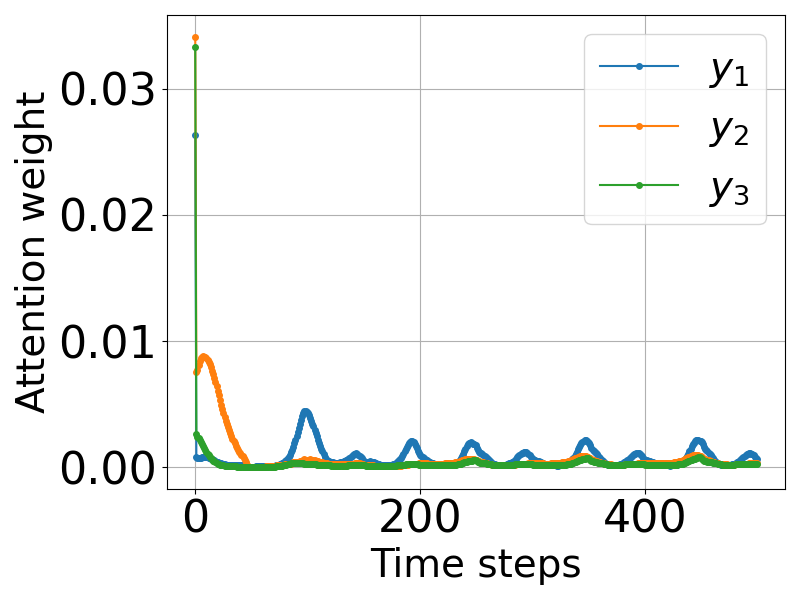}
\end{subfigure}
\begin{subfigure}{0.31\textwidth}
\includegraphics[width=0.93\linewidth]{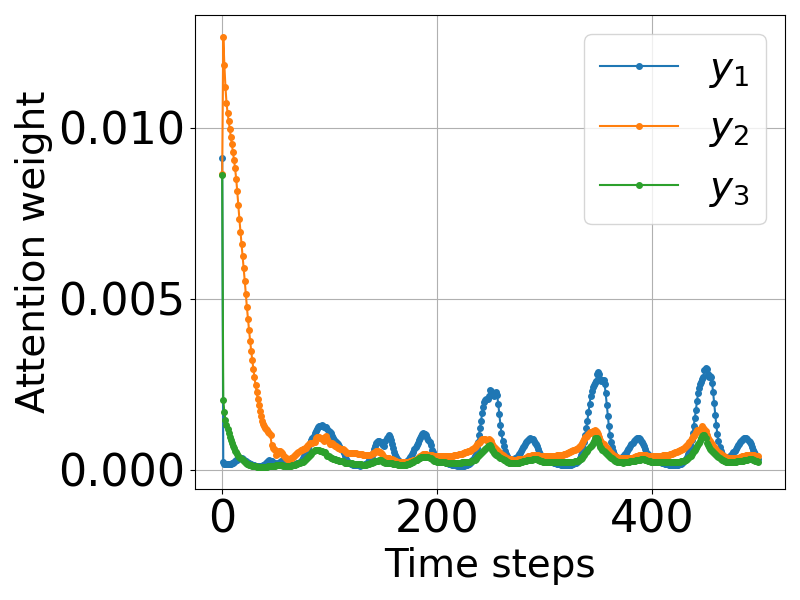}
\end{subfigure}
\begin{subfigure}{0.31\textwidth}
\includegraphics[width=0.93\linewidth]{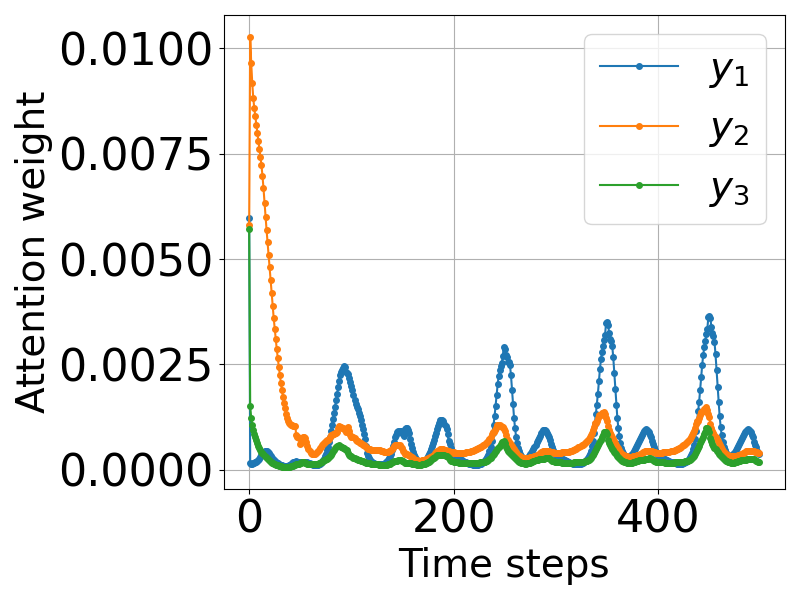}
\end{subfigure}
\caption{}
\label{fig:mask_Ferr_t1}
\end{subfigure}
\caption{Coupled electrochemical kinetics and diffusion model: (a) The upper block (containing time steps up to $33$ time instances) of the block-wise masked attention weight matrix $\widetilde{\boldsymbol{A}}$ for predicting the output $\boldsymbol{y}(t, \boldsymbol{\mu})$ at the testing parameter sample set $\boldsymbol{\mu}:\{f, w_r\}=\{29.9218, 3922.5145\}$. The results are obtained from the SE-TFT trained with MAE loss. (b)(c) Temporal pattern of attention weights for predicting each of the three outputs at the final time step of the prediction horizon (the first and the second testing case).}
\end{figure}

\section{Conclusions}
\label{sec:conclusion}

We extended the TFT to a multi-output setting, yielding the spatially-enhanced TFT. The proposed framework is particularly efficient for predicting multiple outputs of parametric dynamical systems with time and/or parameter varying external inputs. Using two loss functions to target different tasks, SE-TFT supports both actual value predictions and quantile forecasting for multiple outputs. In addition, the interpretability of TFT is extended by introducing an interpretable block-wise masked attention weight matrix for SE-TFT, which reveals the internal (spatial) relations between the multiple outputs evolving along time. These attention patterns are visualized for each numerical example. Moreover, we show the influences of the parameters, inputs, and past outputs on the final prediction by listing their weights computed from variable selection layers. In fact, the variable selection layer is a unique feature of TFT as a transformer model. It further improves the interpretability of TFT, and in turn, that of SE-TFT. As a result, the spatially-enhanced interpretability distinguish SE-TFT from existing transformer-based methods. As compared to the original TFT model, the proposed SE-TFT now supports longer-term rollouts for multiple targets predictions via autoregression.

The numerical results show that SE-TFT can be trained effectively with a moderate amount of data. In the testing phase, SE-TFT requires only the initial condition (or a very short history over the first few time steps) to generate long-horizon forecasts at unseen parameter samples. The resulting predictions are sufficiently accurate, and the spatially‑enhanced interpretability provides a reasonable and meaningful basis for explaining hidden features of the underlying systems.

\section*{Data and code availability}
\addcontentsline{toc}{section}{Data and code availability}
Data and code are available on Zenodo - \url{https://doi.org/10.5281/zenodo.18772646}.

\section*{CRediT authorship statement}
\addcontentsline{toc}{section}{CRediT authorship statement}

Conceptualization: Shuwen Sun, Lihong Feng, Peter Benner; Methodology:  Shuwen Sun, Lihong Feng; Formal analysis and investigation: Shuwen Sun; Writing - original draft preparation: Shuwen Sun; Writing - review and editing:  Shuwen Sun, Lihong Feng, Peter Benner; Supervision: Peter Benner.

\section*{Declarations of interest}
\addcontentsline{toc}{section}{Declarations of interest}
The authors declare that they have no known competing financial interests or personal relationships that could have appeared to influence the work reported in this paper.

\section*{Acknowledgments}
We thank Dr. Tanja Vidakovic-Koch and Dr. Tamara Mili\v{c}i\'{c} from Max Planck Institute for Dynamics of Complex Technical Systems, Germany for providing the coupled electrochemical kinetics and diffusion model. This research was partly supported by the International Max Planck Research School for Advanced Methods in Process and Systems Engineering (IMPRS ProEng), Magdeburg, Germany.

\bibliographystyle{abbrv}
\bibliography{refs}
\end{document}